\documentclass[12pt]{article}

\usepackage{arxiv}        

\usepackage[dvipdfmx]{graphicx}
\usepackage{enumitem}
\usepackage{amsmath}
\usepackage{amssymb}
\usepackage{cite}
\usepackage{subfigure}
\usepackage{algorithm}
\usepackage{algorithmic}
\usepackage{bm}
\usepackage{bbm}
\usepackage{url}
\usepackage{mathtools} 
\mathtoolsset{
  showonlyrefs,
  showmanualtags,
}

\allowdisplaybreaks  % 
\newcommand{\bt}{\bm{\theta}}
\newcommand{\bx}{\bm{x}}

\newcommand{\dd}{\mathrm{d}}

\newcommand{\out}{\mathrm{out}}

\newcommand{\bR}{\mathbb{R}}

\newcommand{\cY}{\mathcal{Y}}
\newcommand{\cX}{\mathcal{X}}

\newcommand{\cH}{\mathcal{H}}

\newcommand{\cP}{\mathcal{P}}

\newcommand{\argmax}{\mathop{\rm arg~max}\limits}
\newcommand{\argmin}{\mathop{\rm arg~min}\limits}

\newcommand{\lan}{\langle}
\newcommand{\ran}{\rangle}

\newtheorem{theo}{Theorem}%[section]
\newtheorem{prop}[theo]{Proposition}
\newtheorem{proof}{Proof}
\newtheorem{ex}{Example}
\newtheorem{defi}{Definition}%[section]

\newcommand{\appropto}{\mathrel{\vcenter{
  \offinterlineskip\halign{\hfil$##$\cr
    \propto\cr\noalign{\kern2pt}\sim\cr\noalign{\kern-2pt}}}}}

\title{Active Learning by Query by Committee with Robust Divergences}

\author{Hideitsu Hino\\
            The Institute of Statistical Mathematics, Tokyo 190-8565, Japan\\
            RIKEN AIP, Tokyo 103-0027, Japan\\
  \texttt{hino@ism.ac.jp}
\And
        Shinto Eguchi\\
                    The Institute of Statistical Mathematics, Tokyo 190-8565, Japan
}

\begin{document}
\maketitle

\maketitle

\begin{abstract}
Active learning is a widely used methodology for various problems with high measurement costs. In active learning, the next object to be measured is selected by an acquisition function, and measurements are performed sequentially. The query by committee is a well-known acquisition function. In conventional methods, committee disagreement is quantified by the  Kullback--Leibler divergence. In this paper, the measure of disagreement is defined by the Bregman divergence, which includes the Kullback--Leibler divergence as an instance, and the dual $\gamma$-power divergence. As a particular class of the Bregman divergence, the $\beta$-divergence is considered. By deriving the influence function, we show that the proposed method using $\beta$-divergence and dual $\gamma$-power divergence are more robust than the conventional method in which the measure of disagreement is defined by the Kullback--Leibler divergence. Experimental results show that the proposed method performs as well as or better than the conventional method.
% \PACS{PACS code1 \and PACS code2 \and more}
% \subclass{MSC code1 \and MSC code2 \and more}
\end{abstract}

\section{Introduction}
Supervised learning, a typical machine learning problem set-up, is the problem of approximating the input--output correspondence given a large number of input and output data pairs in advance. The greater the number of input-output data available for model construction, the more accurately the input-output relationship can be approximated, but the cost of obtaining appropriate outputs for the inputs can be significant. For example, if the environment in which crops are grown (e.g., the average temperature for each month, type and amount of fertilizer to be administered, and weather conditions) is used as the input and property (e.g., sugar content) of a particular crop as the output, it will take several months to several years to obtain the output for a particular input. In another example, researchers have very limited time to use the synchrotron radiation experimental facilities where the X-ray spectrum measurement experiments described below are conducted. Although it is possible to plan particular measurements, which are considered as the input in this case, the cost of obtaining the corresponding output is high, and it is necessary to consider ways to extract the maximum amount of information with as few measurements as possible. There is a methodology called experimental design~\cite{box2005statistics}, which is a technique to carefully design the types and values of input variables and the number of measurements required before conducting an experiment.
On the other hand, when a certain amount of data (input--output pairs) has already been observed and a predictive model has been built using it, the methodology to automatically select the next sample to annotate to maximally improve the predictive accuracy is called active learning~\cite{Settles2010,DBLP:journals/corr/abs-2012-04225}. It is theoretically known that the probability of an incorrect prediction of the response variable for an unknown input (generalization error) can be reduced by appropriately selecting the examples (samples) to be used to train the predictor through active learning. Active learning is widely used in practical applications~\cite{Ueno2018,Terayama2019}, and theoretical analysis has also been conducted~\cite{Balcan2006,Dasgupta2005,AwastiACM17,DBLP:conf/aistats/IshibashiH20}.

Information geometry is a methodology that treats parametric models of probability distributions with a geometric approach~\cite{amari2012differential}. It enables the analysis of statistical inference problems using the tools of differential geometry, and it is used to elucidate the structure of information not only in statistics, but also in various other fields~\cite{10.5555/3019383}. Divergence functions, which quantify the degree of discrepancy between probability distributions, play an essential role in information geometry~\cite{KE2022}, and remarkable results in robust statistics have been obtained, for example, through information geometric analysis with a particular class of divergence functions~\cite{10.1016/j.jmva.2008.02.004}. As a complementary approach to conventional theoretical development and practical algorithms, we consider an active learning algorithm from the viewpoint of information geometry. We consider a sequential active learning problem with a particular acquisition function and a generalized linear model. An intuitively comprehensive picture of the sequential sample selection procedure is provided by considering an input vector as an element of the (algebraic) dual space to the parameter space. On the basis of the geometric formulation of sample selection for an active learning algorithm, robust variants of the selection procedure with a $\beta$ divergence and dual $\gamma$-power divergence are proposed. Proofs for theorems and propositions are deferred to the appendix section for the sake of readability. 

\section{Active Learning}
In this section, the setup of active learning is introduced, and an information geometric perspective of the procedure of selecting a new sample to be measured is considered.

Let $X \in \cX \subseteq \bR^{d}$ be an input variable and $Y \in \cY$ be an output variable, where $\cY$ is a subset of $\bR$ for regression problems and is $\{+1,-1\}$ for discrimination problems. The realizations of the random variables $X$ and $Y$ are denoted as $\bx$ and $y$, respectively. The function $h : \cX \to \cY$, which predicts the response variable from the explanatory variable, is called a hypothesis or predictor. The set of all possible hypotheses is represented by $\cH$. 

The probability density function $p(y|\xi)$ or $p(y|\bm{x};\bt)$ is sometimes abbreviated by its parameter $\xi$ or $\bt$. Let $\mathcal{P}$ be the space of all Radon--Nikod\'{y}m derivatives with a common support, which are dominated by a $\sigma$-finite measure $\Lambda$ on $\mathbb{R}^{d}$. We typically consider most cases where $\Lambda$ is fixed by the Lebesgue measure or the counting measure so that $\mathcal{P}$ is the space of all probability density functions or that of all probability mass functions.

\subsection{Sequential Observation for Generalized Linear Model}
Suppose we have a small initial training dataset $S_0 = \{(\bm{x}_i, y_i)\}$, and the initial predictive model $y = h_0(\bm{x})$ is trained with $S_0$. 
In active learning, a {\it{learner}} is supposed to select a sample $\bx$ for which the value of the corresponding output variable $y$ is unknown by some criteria, thereby obtaining the value of $y$. The function that returns the value of the explanatory variable for $\bx$ is often called the {\it{oracle}}. In statistics, for most sequential designs, a setting is assumed in which observation points can be freely selected according to some standard; this is called membership query synthesis in the context of active learning~\cite{Angluin1988}. On the other hand, in the literature on active learning, it is often assumed that the learner has access to a set of pooled unlabeled samples denoted by $\cX_p$ and selects one sample from $\cX_p$ in one iteration of the active learning procedure on the basis of the value of an {\it{acquisition function}} $a(\bm{x})$. Then, the output value $y$ for the chosen sample $\bx$ is measured, and the dataset for learning the predictive model is updated as $S_{t+1} = S_{t} \cup \{(\bx,y)\}$. We follow this problem setting. 

Given a set of input and output pairs $S_t = \{ (\bm{x}_i,y_i)\}_{i = 1,\dots,n_t}$, we define $X_t = (\bm{x}_1,\dots,\bm{x}_{n_t})^{\top} \in \mathbb{R}^{n_t \times d}$ as the design matrix and a vector of realizations of output $\bm{y}=(y_1,\dots,y_{n_t})$. We consider the joint distribution $p(\bm{y} | X_t, \bm{\xi})$ of $\bm{y}$ parameterized by $\bm{\xi} \in \bR^m$ and estimate $\bm{\xi}$ by the maximum likelihood method.
An exponential family is a broad class of statistical models, which includes Gaussian distribution and Poisson distribution, for example. A generalized linear model (GLM;~\cite{10.1214/ss/1177013604}) considers that each output $y$ is assumed to be generated from a particular distribution in an exponential family. Statistical inference on the GLM is investigated from the viewpoint of information geometry~\cite{HiroseKomaki2010,EGUCHI202115}. There are several equivalent representations for GLM, and we adopt the following form:
\begin{align}
    p(y|\xi(\bm{x})) = &
    \exp
    \left(
    \frac{ y \xi(\bm{x}) - \psi(\xi(\bm{x}))}{\phi} + c(y,\phi)
    \right),
    \label{eq:GLM}
\end{align}
where $\phi$ is a dispersion parameter and is assumed to be known in this paper. The function $\psi(\xi(\bx))$ is the cumulant generating function. 
Here, $\xi \in \mathbb{R}$ is the canonical parameter of the distribution, but in the framework of the generalized linear model, we further introduce the linear predictor $\xi(\bx) = h(\bx;\bt) = \lan \bt, \bx \ran$ and consider $\bt$ as the target of estimation, where $\lan \; \cdot \;, \; \cdot \; \ran$ is the Euclidean inner product, and accordingly, $p(y|\xi(\bx)) = p(y|\bx;\bt)$. 
As a basic requirement, $\xi$ is in $(-\infty,\infty)$, so that the linear model $\xi = \lan \bt,\bx \ran$ is always well-defined for any regression parameter $\bt \in \bR^{d}$.

\begin{ex}{Multiple Regression}\\
In the standard multiple linear regression in which the output $y$ follows a Gaussian distribution with the mean $\xi$ and variance $\sigma^2$, $\phi = \sigma^2$, and $c(y,\phi) = - \frac{y^2}{\phi} - \log (2\pi \phi)^{1/2}$. The cumulant generating function is $\psi(\xi) = \frac{1}{2}\xi^2$, and the probability density function is
\begin{align}
\begin{aligned}
    p(y|\xi(\bx)) = &
    \frac{1}{\sqrt{2\pi \sigma^2}}
    \exp 
    \left(
    - \frac{1}{2 \sigma^2} (y-\xi)^2
    \right)\\    =&
    \exp 
    \left(
    \frac{y \lan \bt,\bx \ran - \frac{1}{2} \lan \bt,\bx \ran^2}{\phi} - \frac{y^2}{2 \phi} - \log (2\pi \phi)^{1/2}
    \right)
    .
    \label{eq:normalPDF}
\end{aligned}
\end{align}
\end{ex}

\begin{ex}{Logistic Regression}\\
In logistic regression, the output $y$ follows a binomial distribution with parameters $\xi$, $\phi = 1$, and $c(y,\phi) = 0$. The cumulant generating function is $\psi(\xi) = \log (1+e^{\xi})$; hence, the probability function is 
\begin{align}
    p(y|\xi(\bx)) = \exp( y \xi - \log (1+e^{\xi})).
    \label{eq:logPMF}
\end{align}
\end{ex}

Since the generalized linear model~\eqref{eq:GLM} is determined by design 
$X_{t-1}$ and parameter $\bt$, the model manifold is denoted as
\begin{align}
    M_{X_{t-1}} = 
    \{
    p(\bm{y}|X_{t-1};\bt) 
    \}.
\end{align}

In the framework of active learning, we assume, by using $\bm{y}$ and $ X_{t-1}$, that an estimate of parameter $\hat{\bm{\theta}}_{t-1}$ has been obtained. Then, by using the information contained in $\bm{y}, X_{t-1}$ and $\hat{\bt}_{t-1}$, we explore a point $\bx_{t} \in \cX_p$. Considering the fact that the canonical parameter for the generalized linear model $\xi$ is expressed by $\lan \bm{\theta}, \bm{x} \ran = \xi (\bm{x})$, we consider fixing $\bt = \hat{\bt}_{t-1}$ and $\bx$ as the element of the algebraic dual space of $\bt \in \Theta$. Then, the problem of exploring the sample point $\bx \in \cX_p$ in active learning is considered as the problem of finding $\bx$ by maximizing an acquisition function characterized by $p(y|\bx; \hat{\bt}_{t-1})$.
The dimensionality of joint distribution is $t-1$ in $M_{X_{t-1}}$ and $t$ in $M_{X_t}$, while the dimensionality of parameter $\bt$ remains the same. The parameter $\bt$ specifies a point of the model manifold $M_{X_{t}}$ or $M_{X_{t-1}}$. 
From this perspective, the updated dataset $S_{t} = S_{t-1} \cup \{(\bx_{t},y_t)\}$ in the process of active learning is regarded as the extension of the model space $M_{X_{t-1}}$ to $M_{X_{t}}$ as schematically depicted in Figure~\ref{fig:Ext}.
We consider an active learning problem that explores an additional observation point $\bt_{t}$ based on the current model parameter $\hat{\bt}_{t-1}$. 
In Fig.1, the dotted arrow connects the different model spaces, and the solid arrow connects model parameters within the same model space $M_{X_t}$. The cross mark specifies a point in $M_{X_{t-1}}$ while open circles specify points in $M_{X_{t}}$. 
Note that the parameter $\hat{\bt}_{t-1}$ corresponds to a point obtained by MLE with $S_{t-1}$ in $M_{X_{t-1}}$. Additional observation $(\bt_t,y_t)$ defines an updated model space $M_{X_t}$. The parameter $\hat{\bt}_{t-1}$ specifies a point in the updated model space $M_{X_t}$, but it is not the parameter $\hat{\bt}_{t}$ obtained by MLE using the updated dataset $S_t$ in general.

\begin{figure}[ht]
\centering
  \includegraphics[width=60mm]{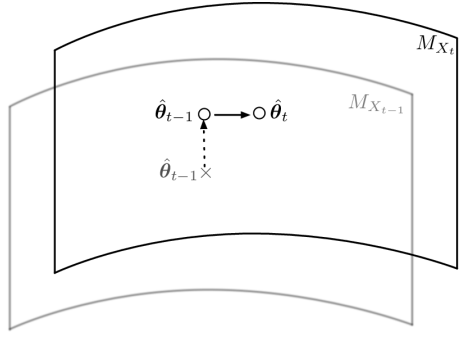}
  \caption{The model space $M_{X_{t-1}}$ corresponding to the design $X_{t-1}$ up to the step $t-1$ of active learning, and the extended space $M_{X_t}$ based on the new observation $\bx_t$.} 
  \label{fig:Ext}
\end{figure}

\subsection{Acquisition Function}
The design of the acquisition function is one of the central issues in active learning studies.  Several acquisition functions aim to quantify the difficulty of label prediction in some way and actively incorporate difficult-to-predict samples into learning. This approach of designing acquisition functions is called the uncertainty-based approach.
As another approach, it is reasonable to sample inputs so as to reflect the distribution of explanatory variables, and methods that are based on the idea of annotating representative samples~\cite{Nguyen2004,Sener2018} have been proposed.
In addition, several methods have recently been proposed to learn acquisition functions according to the environment and data~\cite{Konyushkova2017,Haussmann2019,DBLP:journals/npl/TaguchiHK21}.

In this work, we consider the uncertain-based approach. In particular, we adopt a simple and intuitive method for quantifying the uncertainty called the query by committee (QBC).

\subsection{Query by Committee}
\label{sec:QBC}
One of the criteria for selecting new observations in active learning is the query by committee (QBC;~\cite{Seung1992}). This is an approach that selects the sample on which there will be the most disagreement in a sense of a consensus of multiple predictive models. 
Each time a new sample $\bx \in \cX_p$ or query is issued, {\it{committee members}} vote on the response $y$ for the query $\bx$.
Various methods have been proposed and discussed in relation to ensemble learning~\cite{Freund1997} and the resulting reduction of version space~\cite{Gilad-Bachrach2005}. As a simple representative method, the following procedure is proposed in \cite{10.5555/645527.757765}:
\begin{enumerate}
    \item Learn $C$ predictive models with different parameters $\bt_{c,t-1},c=1,\dots,C$ by, for example, Bagging~\cite{breiman96}.
    \item Select a sample from the pool as $\bx_t = \argmax_{\bx \in \cX_p} a(\bx)$ by using the acquisition function
    \begin{align} 
        a_0(\bx) =\sum_{c=1}^{C} w_c D_0(p(Y|\xi_{c,t-1}(\bx)), p(Y|\bar{\xi}(\bx))),
        \label{eq:ac}
    \end{align}
    where $D_0$ is the Kullback--Leibler (KL) divergence, and $\bar{\xi}(\bx)$ is the {\it{consensus model parameter}} defined later, and $\xi_{c,t-1}(\bx) = \lan \bt_{c,t-1}, \bx \ran$. 
    Measure the response $y$ corresponding to the selected $\bx_t$ and denote it as $y_t$. 
    \item Update the training dataset $S_{t} = S_{t-1} \cup \{(\bx_{t},y_{t})\}$.
\end{enumerate}
In Eq.~\eqref{eq:ac}, the divergences are mixed with the mixing weight $w_c, c=1,\dots, C$ where $w_c \geq 0$ and $\sum_{c=1}^{C} w_c = 1$. The weight $w=(w_1,\dots,w_C)$ reflects the reliability of committee members, and can be fixed in advance or determined during the learning procedure for a committee member. In this work, we consider $w_c = 1/C$ and fixed throughout the active learning process.

The consensus model parameter $\bar{\xi}(\bx)$ is defined by the minimizer \begin{align}
    \bar{\xi}(\bx) = \argmin_{\xi(\bx)}
    \sum_{c=1}^{C}
    w_c 
    D_{0}(
    p(Y|\xi(\bx)),
    p(Y|\xi_c(\bx)))
    \label{eq:consKL}
\end{align}
of the weighted sum of $e$-projections from committee models $p(y|\xi_c(\bx)) = p(y|\bx;\bt_c)$ as shown in Figure~\ref{fig:consKL}. The minimizer~\eqref{eq:consKL} is called the $e$-mixture of models~\cite{716791,DBLP:journals/neco/TakanoHAM16}.
\begin{figure}[ht]
\centering
  \includegraphics[width=60mm]{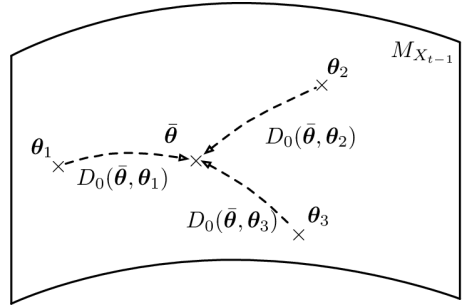}
  \caption{Consensus model defined by a minimum of summation of $e$-mixtures. The parameter $\xi$ of exponential family is specified via the parameter $\bt$ of the linear predictor. Dashed arrows indicate $e$-projections.}
  \label{fig:consKL}
\end{figure}

The consensus model is explicitly obtained by minimizing
\begin{align}
    \notag
    J(\xi) =&
    \sum_{c=1}^{C} w_c D_{0}(
    p(Y|\xi(\bx)),
    p(Y|\xi_c(\bx)))
    =
    \sum_{c=1}^{C} w_c \mathbb{E}_{\xi}
\left[ 
\frac{ 
Y \xi - Y \xi_c - \psi(\xi) + \psi(\xi_c)
}{
\phi
}
\right]\\ 
=&
\frac{1}{\phi} 
\sum_{c=1}^{C} w_c \{
\psi^{\prime}(\xi) \xi - \psi^{\prime}(\xi) \xi_c - \psi(\xi) + \psi(\xi_c)
\},
\end{align}
where we used $\psi^{\prime}(\xi) = \mathbb{E}_{\xi}[Y]$.
Then, solving $ \frac{\partial J(\xi)}{\partial \xi} =
\psi^{\prime \prime} (\xi)\sum_{c=1}^{C} w_c(
\xi - \xi_c)
=0$, we have $\bar{\xi} = \sum_{c=1}^{C} w_c \xi_c$.

The acquisition function in QBC is defined by the summation of KL divergences from models $p(y|\xi_c(\bx))$ to the consensus model $p(y|\bar{\xi}(\bx))$ as depicted in Figure~\ref{fig:acFunc}.
Intuitively, the query with the most split votes would be worth querying the oracle because of the high uncertainty. To quantify the diversity of this vote, the divergence from the average model is considered. When the sum of the divergences from the mean is large, the individual committee member's disagreement is considered to be large. In the original work~\cite{10.5555/645527.757765}, the acquisition function~\eqref{eq:ac} is defined as the sum of the KL-divergences from committee members to the consensus model~\eqref{eq:consKL}. It is also possible to consider the sum of the divergences from the consensus model to the committee members, which is used for defining the consensus model. In this paper, we only consider the case~\eqref{eq:consKL} following the original definition in~\cite{10.5555/645527.757765}, but the sum of the divergences from the consensus model to the committee members gives similar results.
\begin{figure}[ht]
\centering
  \includegraphics[width=100mm]{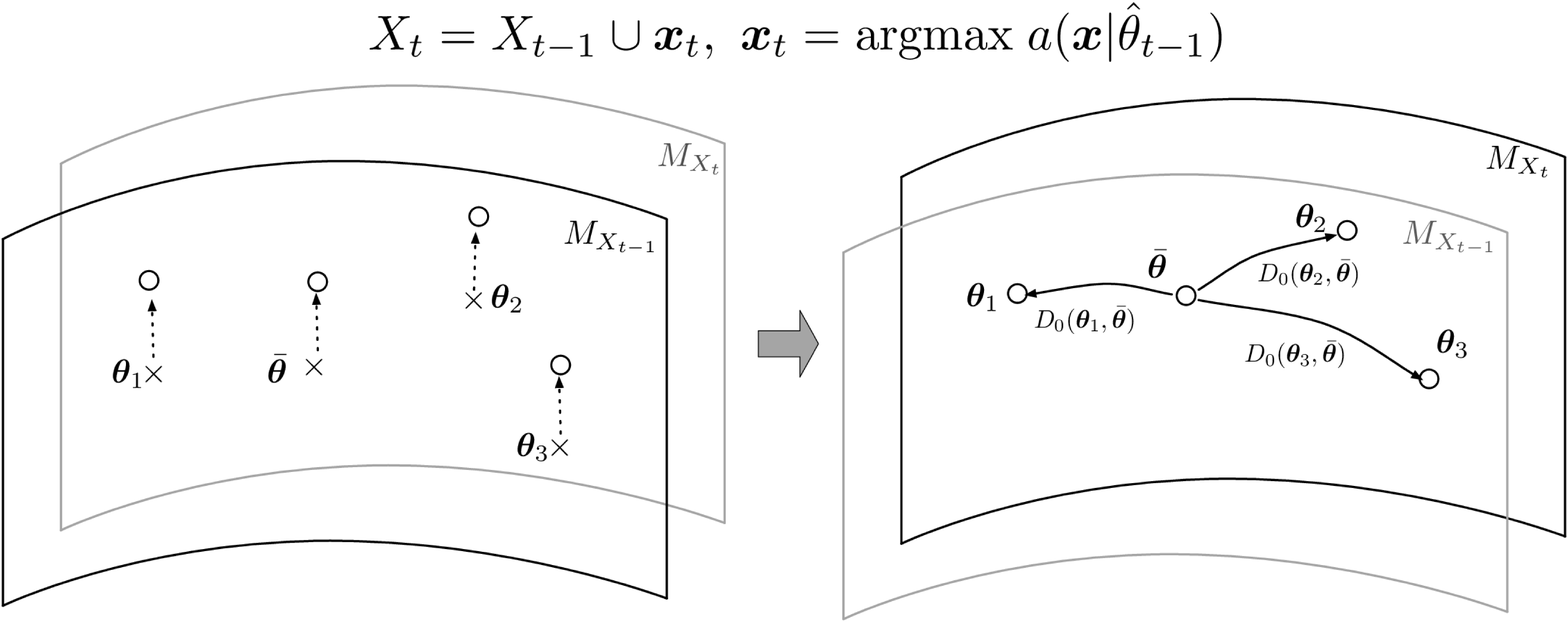}
  \caption{The procedure corresponds to the acquisition function in QBC active learning. A new sample $x_t$ is selected by maximizing KL divergences from the consensus model to committee models.}
  \label{fig:acFunc}
\end{figure}

Note that the maximization of Eq.~\eqref{eq:ac} is often not a well-defined optimization problem when $\bx$ freely takes any value in $\cX$. In fact, the domain of $\bx$ is a set of candidate measurement points $\cX_p$ given in advance, the so-called pool, and maximization is carried out by evaluating all the elements of this set. Therefore, there is no need for mathematical optimization, and the maximization is simply carried out by evaluating the acquisition function $a_0(\bx)$.

The consensus model and the acquisition function are originally defined on the basis of the KL divergence, which is vulnerable to outliers. In the QBC procedure explained above, committee models are trained by using a small amount of training data. If we use Bagging to construct various committee members, the situation would be worse, and it is highly possible that some of the committee members behave as outliers. To alleviate this problem, we consider using robust divergence measures instead of a standard KL divergence for computing the consensus model and defining the acquisition function.

%%%%%%%%%%%%%%%%%%%%%%%%%%%
\section{Divergence Functions}
Divergence function is an index to measure the discrepancy between two probability density functions. It plays a central role in integrating statistics, information theory, statistical physics, and machine learning with many other fields. The most popular divergence function is the KL divergence. In this section, we consider two classes of alternative divergences.

\subsection{Bregman divergence}
Let $U$ be a monotonically increasing convex function on $\bR$ and $u$ be the derivative of $U$. We define $U^{\ast}(\zeta) = \sup_{z \in \mathbb{R}} \{z \zeta -U(z)\}$, that is, the Legendre transform of $U$, and $u^{\ast} = u^{-1}$ as the derivative of $U^{\ast}$. We consider transforming the function $f$ by $u^{\ast}(f)$, and denote the transformed function as $\breve{f} = u^{\ast}(f)$, which is called the $u$-representation of the function $f$. Then, the Bregman potential between two functions $f$ and $g$ is defined as
\begin{align}
    d_U(f,g) = U^{\ast}(f) + U(\breve{g}) - f \breve{g},
\end{align}
and the Bregman divergence~\cite{BREGMAN1967200} is defined as
\begin{align}
    D_{U}(p,q) = 
    \int d_{U}(p(y),q(y)) \dd \Lambda (y) = \int d_{U}(p,q) \dd \Lambda,
\end{align}
where $p$ and $q$ are the probability density or probability mass functions. 
Note that we omit the integral variable $y$ for notational simplicity. Then, the $u$-cross entropy and $u$-entropy are defined as
\begin{align}
    C_U(p,q) =&
    \int U(\breve{q}) \dd \Lambda - \int p \breve{q} \dd \Lambda,
    \\
    H_U(p) =&
    \int U(\breve{p}) \dd \Lambda - \int p \breve{p} \dd \Lambda,
\end{align}
respectively. 
Using these entropies, we define the Bregman divergence or the $u$-divergence from $p$ to $q$ as
\begin{align}
    D_U(p,q) = C_U(p,q) - H_U(p).
\end{align}
The most popular convex function $U$ and its related functions for the Bregman divergence would be the exponential function, which leads to the Kullback--Leibler divergence where
\begin{align}
\begin{aligned}
    U(z) &= \exp(z), 
    & U^{\ast}(\zeta) &= 
    \zeta ( \log \zeta -1 ),
\\
    u(z) &= \exp(z),&
    u^{\ast}(\zeta) &= \log \zeta.
\end{aligned}
\end{align}

The Euclidean distance is recovered with
\begin{align}
    \begin{aligned}
        U(z) &= \frac{1}{2}z^2, & U^{\ast}(\zeta) &= \frac{1}{2}\zeta^2,\\
        u(z) &= z, & u^{\ast}(\zeta) &= \zeta.
    \end{aligned}
\end{align}

Other important examples include the $\eta$-type with $\eta \geq 0$
\begin{align}
\begin{aligned}
    U(z) &= \exp(z) + \eta z, &
    U^{\ast}(\zeta) &= (\zeta - \eta) \{ \log (\zeta - \eta) +1\},\\
    u(z) &= \exp(z) + \eta, &
    u^{\ast}(\zeta) &= \log (\zeta - \eta),
\end{aligned}
\end{align}
and the $\beta$-type with $\beta \geq 0$
\begin{align}
\begin{aligned}
    U(z) &= \frac{1}{\beta+1} (\beta z + 1)^{\frac{\beta+1}{\beta}}, 
    & U^{\ast}(\zeta) &= 
    \frac{\zeta^{\beta+1}}{\beta(\beta+1)} - \frac{\zeta}{\beta},\\
    u(z) &= (\beta z + 1)^{1/\beta},&
    u^{\ast}(\zeta) &= \frac{\zeta^{\beta} -1}{\beta}.
\end{aligned}
\end{align}
Both the $\eta$-type and $\beta$-type functions lead to robust estimators. In this work, we concentrate on the $\beta$-type and only consider the $\beta$-divergence
\begin{align} \notag
    D_{\beta}(p,q)
    =&
    \frac{1}{\beta+1} \int q^{\beta+1} \dd \Lambda - 
    \frac{1}{\beta +1} \int p^{\beta+1} \dd \Lambda -
    \frac{1}{\beta} \int p (q^{\beta} - p^{\beta}) \dd\Lambda\\
    =&
    {\frac{1}{\beta(1+\beta)}}\int p^{\beta+1}{\dd}\Lambda
- \frac{1}{\beta}{\int p q^\beta \dd \Lambda}
+{\frac{1}{\beta+1}}\int q^{\beta+1} \dd \Lambda.
\label{eq:beta-div}
\end{align}
as an instance of the Bregman divergence.

\subsection{Dual $\gamma$-power Divergence}
In \cite{10.1093/biomet/85.3.549}, it is proposed that a class of power divergences, and discussed a robust parameter estimation on the basis of this class. It is shown to be robust to outliers, and its relationship with the pseudo-spherical score was investigated in~\cite{10.1016/j.jmva.2008.02.004,10.3150/13-BEJ557,10.1093/biomet/asv014}.

The standard $\gamma$-power divergence is given by
\begin{eqnarray}
    D_{\gamma}(p, q)=- \frac{\int p q^\gamma \dd \Lambda}{(\int q^{\gamma+1} \dd \Lambda)^{\frac{\gamma}{\gamma+1}}}+\Big(\int p^{\gamma+1} \dd \Lambda\Big)^{\frac{1}{\gamma+1}},
\end{eqnarray}
which satisfies the scale invariance for the second argument as $D_{\gamma}(p, zq)=D_{\gamma}(p, q)$ for any constant $z>0$.

Alternatively, we consider a dual $\gamma$ power divergence as
\begin{align}
\label{gamma-star}
    D^*_{\gamma}(p, q)=- \frac{\int p q^\gamma{\rm d}\Lambda}{(\int p^{\gamma+1} \dd \Lambda)^{\frac{1}{\gamma+1}}}+\Big(\int q^{\gamma+1} \dd \Lambda\Big)^{\frac{\gamma}{\gamma+1}}
\end{align}
for $p$ and $q$ of $\cP$.
Note that $D^*_{\gamma}(zp, q)= D^*_{\gamma}(p, q)$ for any constant $z>0$.
This implies that, if $p$ and $q$ are density functions with a finite mass, then $  D^*_{\gamma}(p, q) \geq0$
with equality $p=zq$.
On the other hand, if $p$ and $q$ are in $\cal P$, then $  D^*_{\gamma}(p, q) =0$ means $p=q$.
It is worth noting that the dual $\gamma$-power divergence is closely related to the $\beta$-divergence defined in Eq.~\eqref{eq:beta-div}. 
Consider a scale adjustment as $\min_{v>0}D_{\beta}(vp, q)$. We observe that
\begin{align}
    \frac{\partial}{\partial v}D_{\beta}(vp, q)={\frac{v^{\beta}}{\beta}}\int p^{\beta+1} \dd \Lambda
- \frac{1}{\beta}{\int p q^\beta \dd \Lambda},
\end{align}
in which the minimizer is given by
$  v^*=\big( \frac{\int p q^\beta \dd \Lambda}
{\int p^{\beta+1} \dd \Lambda}\big)^{\frac{1}{\beta}},$
so that we have the minimum 
\begin{align}
\label{min}
 D_{\beta}(v^*p, q)=
- \frac{1}{\beta+1}\bigg[\frac{\int p q^\beta \dd \Lambda}{\big(\int p ^{\beta+1} \dd \Lambda\big)^{\frac{1}{\beta+1}}}\bigg]^{\frac{\beta+1}{\beta}}
+{\frac{1}{\beta+1}}\int q^{\beta+1} \dd \Lambda.
\end{align}
We note that this divergence is scale-invariant for the first argument.
By a power transformation for two terms on the right side of Eq.~\eqref{min}, we obtain
\begin{align}
\label{dif}
 \bigg[\int q^{\beta+1} \dd \Lambda\bigg]^{\frac{\beta}{\beta+1}}
\geq \frac{\int p q^\beta \dd \Lambda}{\big(\int p ^{\beta+1} \dd \Lambda\big)^{\frac{1}{\beta+1}}}.
\end{align}
Accordingly, we get $D^*_\gamma(p,q)$
by the difference between both sides of \eqref{dif} if $\beta=\gamma$.

\section{Consensus model defined by robust divergences}
The consensus model used in QBC is the model with parameter $\bar{\xi}$ defined in Eq.~\eqref{eq:consKL}, which gives the minimum sum of the KL divergences from the consensus model $\bar{\xi}$ to the committee members $\xi_c, c=1,\dots,C$. The acquisition function~\eqref{eq:ac} is also defined as the KL divergence.

In the problem setting of active learning, the number of samples given initially can be very small. Predictive models fitted using a very small number of samples, which are possibly further reduced by splitting, are likely to be very inaccurate. In the case of a linear model, the fitted coefficients $\bt_c$ can take extremely large values, so that $|\xi_c(\bx)| = |\lan \bt_c, \bx \ran|$ can be very large. This means that the parameter $\xi_c(\bx)$ behaves as an outlier in the construction of the consensus model or in the calculation of the acquisition function. Therefore, instead of the KL divergence, it is reasonable to consider a consensus model and an acquisition function using robust divergences such as the $\beta$-divergence~(Fig.~\ref{fig:consU}), which has a limited effect on the inclusion of outliers. Stable behavior can be expected even in situations where the committee members are not reliable. Unfortunately, the consensus model based on the Bregman divergence is not well-defined in general due to the impossibility of normalization. For such cases, we also consider a dual $\gamma$-power divergence, which provides an explicit consensus model that is well defined irrespective of the distributions to be mixed.

\begin{figure}[ht]
\centering
  \includegraphics[width=80mm]{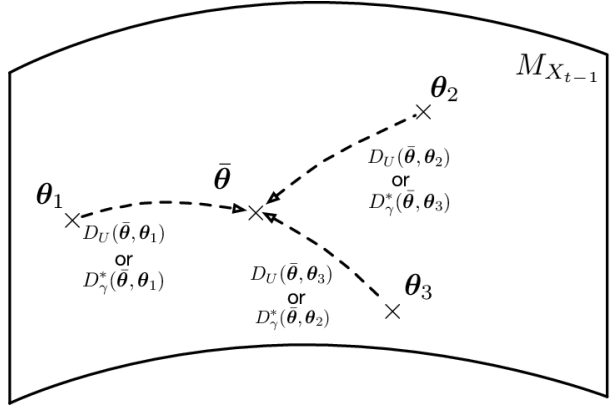}
  \caption{Consensus model based on $u$-mixture or dual $\gamma$-power mixture.}
  \label{fig:consU}
\end{figure}

\subsection{Consensus model with Bregman divergence}
It is nontrivial how to calculate the consensus model with the Bregman divergence defined by using a $\beta$-type  convex function. However, the following theorem provides a way to achieve the minimizer of the sum of Bregman divergences. 

\begin{theo}[Characterization of $u$-mixture~\cite{muratafujimoto2009}]: 
\label{theo:umix}
Let $w = (w_1,\dots,w_C) \in \Delta_C$ be an element of $C$ probability simplex. For probability density functions or probability mass functions $p_c(y), c=1,\dots,C$, consider
\begin{align}
    A_{U}(q;w) = \sum_{c=1}^{C} w_c D_{U}(q,p_c).
    \label{eq:dind}
\end{align}
Then, 
\begin{align}
    \argmin_{q} A_{U}(q;w) = p_{u}(y;w),
\end{align}
where
\begin{align}
    p_u (y;w) 
    =
    u \left( \sum_{c=1}^{C} w_c \breve{p}_{c}(y) - b\right).
    \label{eq:Ucons}
\end{align}
\end{theo}
\begin{proof}
Proof of this theorem is shown in Appendix~\ref{app:A}.
\end{proof}

The model $p_u(y;w)$ is called the {\it{$u$-mixture}} of $p_c(y), c=1,\dots,C$ associated with weight $w$. 
The constant $b$ is a normalizing factor so that $p_{u}$ is a valid probability density or mass function.

\begin{ex}[KL-divergence and geometric mean]
When we consider the KL-divergence as an instance of the $u$-divergence, the geometric mean of $p_{c}(y)$ with the weights $w_{c}$ defined by
\begin{align}
    \bar{p}_G(y;w) =
    e^{-b(w)}\prod_{c=1}^{C} p_{c}(y)^{w_c},
\end{align}
where $b(w) = \log \int \prod_{c=1}^{C}p_c(y)^{w_c}dy$, minimizes the weighted average $A_0(p)=\sum_{c=1}^{C}w_c D_{0}(p,p_c)$ because $A_0(p) \geq A_0(\bar{p}_{G}) = 0.$
\end{ex}
\begin{ex}[Euclidean distance and arithmetic mean]
It is also straightforward to show that the arithmetic mean 
\begin{align}
    \bar{p}_A(y;w) = \sum_{c=1}^{C} w_c p_c(y)
\end{align}
is the minimizer of the weighted sum corresponding to the Euclidean distance. 
\end{ex}

To be concrete, for the $\beta$-divergence,
\begin{align} \notag
&    \sum_{c=1}^{C} w_c \breve{p}(y|\xi_c(\bx)) - b(\bx)\\ 
    =&
    \frac{1}{\beta} 
    \sum_{c=1}^{C} w_c \exp 
    \left\{
    \beta \frac{y \xi_c(\bx) - \psi( \xi_c(\bx))}{\phi} + \beta c(y,\phi)
    \right\} - \frac{1}{\beta} - b(\bx),
\end{align}
hence we have
\begin{align}
\notag
p_u(y|\bx;w) =&
u\left(\sum_{c=1}^{C} w_c \breve{p}(y|\xi_c(\bx)) - b(\bx)
\right)
\\
=&
\left[
\sum_{c=1}^{C} w_c \exp 
    \left\{
    \beta \frac{y \xi_c(\bx)- \psi( \xi_c(\bx) )}{\phi} + \beta c(y,\phi)
    \right\} -\beta b(\bx)
\right]^{\frac{1}{\beta}}.
\label{eq:beta_mix}
\end{align}
We note that in this expression, the consensus model does not have an explicit consensus parameter $\bar{\xi}$; therefore, we will denote the consensus model as $p_u(y|\bx;w)$ instead of $p(y|\bar{\xi}(\bx);w)$ henceforth. 
This expression contains a normalization factor $b(\bx)$, which depends on the input variable $\bx$; therefore, the analytical calculation is prohibitive.

We characterize the robustness of the $u$-mixture in terms of the influence function~\cite{rousseeuw2011robust}. 
Consider the mixture of models $p_c(y) = p(y;\xi_c)$ with respect to the $u$-divergences with a weight $w_c$.
We will omit $\bx$ for notational simplicity. 
Then, we define an $\epsilon$-contamination weight operator $W_\epsilon$
\begin{align}
    W_{\epsilon}[p] = &
    \int  \left\{ (1-\epsilon) \sum_{c=1}^{C} w_c \delta(\xi - \xi_c) + \epsilon \delta(\xi - \xi_{\out}) \right\} p(y;\xi) \dd \xi, 
\end{align}
where $\delta$ denotes the Dirac measure degenerated at zero. For notational simplicity, we write $p(y;w(\epsilon))= W_{\epsilon}[p]$. %HERE: It acts on probability density or probability function, A_U, A_gamma, etc, abuse of notation.

The outlier $\xi_{\out}$ is a parameter value far from other committee members $\xi_c$. In GLM, the parameter $\xi$ is composed of $\bt$ and $\bx$; hence there are two possible reasons for the outlying point. One is the outlying input point $\bx$, and the other is the outlying regression coefficient $\bt$. In this work, the pool $\mathcal{X}_p$ is fixed and $\|\bx\|$ is bounded, but it is meaningful to consider the situation that the parameter $\xi = \lan \bt, \bx \ran$ takes a very large value. The regression coefficient $\bt$ is, in principle, not bounded and it may make $\xi$ arbitrary value.

With this $\epsilon$-contamination weight operator, we replace the weighted sum operation such as $(1-\epsilon)\sum_{c=1}^{C} w_c p(y;\xi_c) + \epsilon p(y; \xi_{\out})$ by $W_{\epsilon}[p] = p(y;w(\epsilon))$, 
%\begin{align}
%\int \dd \xi \left\{ (1-\epsilon) \sum_{c=1}^{C} w_c \delta(\xi - \xi_c) + \epsilon \delta(\xi - \xi_{\out}) \right\} p(y;\xi),
%\end{align}
which is reduced to $\sum_{c=1}^{C} w_c p(y;\xi_c)$ when $\epsilon= 0$. 

We focus on the contamination of the outlier model $p(y;\xi_{\out})$ in the mixture and consider its influence on the minimizer $p_u(y;w)$ of $A_U(q;w)$.

\begin{defi}[Influence function of $u$-mixture]
The influence function of the minimizer $p_{u}(y;w)$ of the weighted sum of the $u$-divergences $A_U(q;w) = \sum_{c=1}^{C} w_c D_{U}(q,p_c)$ for the outlier $\xi_{\out}$ is defined as
\begin{align}
    \mathrm{IF}(p_u(y;w),\xi_{\out}) 
    =
    \lim_{\epsilon \to 0}
    \frac{ 
    p_{u}(y;w(\epsilon)) - p_{u}(y;w)
    }{
    \epsilon
    }.
\end{align}
\end{defi}

With this influence function, we can characterize the robustness of the $u$-mixture $p_u(y;w)$ of distributions in the exponential family. 
\begin{prop}
\label{prop:consIFbeta}
The influence function of the $u$-mixture of exponential family distributions is unbounded when the $u$-divergence is the Kullback--Leibler divergence. The influence function of the $u$-mixture of exponential family distributions is bounded when the $u$-divergence is the $\beta$-divergence with $\beta > 0$.
\end{prop}
\begin{proof}
Proof of this proposition is shown in Appendix~\ref{app:B}.
\end{proof}

\subsection{Consensus model with dual $\gamma$-power divergence}
If the domain of $u^{\ast}$ is restricted to a subset of $\mathbb{R}$,
then there does not exist such a normalizing constant $b$ in Eq.~\eqref{eq:Ucons}.
For example, when $D_U$ is the $\beta$-power divergence, the minimizer is written as
\begin{eqnarray}
   p_u(y;w) =\Big(\sum_{c=1}^C w_c p_c(y)^\beta -b\Big)^{\frac{1}{\beta}}.
\end{eqnarray}
If $p_c(y)>0$ for all $y$ of $\mathbb R$, $\lim_{\>|y|\rightarrow\infty}p_c(y)=0$.
Hence, it must be $\beta>0$ and $b\leq0$ since $\lim_{\>|y|\rightarrow\infty}p_u(y;w)=(-b)^{\frac{1}{\beta}}$.
If $b<0$, $p_u(y;w)$ is not integrable because $p_u(y;w)\geq (-b)^{\frac{1}{\beta}}$.
Note that this problem will occur unless the support of $p_c(y)$ is finite-discrete.

To solve the problem above, we consider the dual $\gamma$-power divergence $D^{\ast}_{\gamma}(p,q)$ defined in~\eqref{gamma-star}. We introduce a simple result for the minimum dual $\gamma$-mixture.
\begin{prop}
\label{prop:gammamix}
Let $A_\gamma(q)=\sum_{c}w_c D^*_\gamma(q,p_c)$, where $D_\gamma^*(q,p_c)$ is defined in \eqref{gamma-star}.
Then the minimizer of $A_\gamma(q)$ in $q$ of $\cal P$ is given by
\begin{align}\label{gamma}
    p_{\gamma}(y;w)=\frac{1}{z(w)}\Big( \sum_{c=1}^C w_cp_c(y)^\gamma\Big)^{\frac{1}{\gamma}},
\end{align}
where
\begin{align}
{z(w)}=\int \Big( \sum_{c=1}^C w_cp_c(y)^\gamma\Big)^{\frac{1}{\gamma}}d\Lambda (y).
\label{eq:normFactorGamma}
\end{align}
\end{prop}
\begin{proof}
Proof of this proposition is shown in Appendix~\ref{app:A3}.
\end{proof}
The model $p_{\gamma}(y;w)$ is called the {\it{dual $\gamma$-mixture}} of $p_c(y), c=1,\dots,C$ associated with weight $w$. 

We now focus on the behaviors of the minimizers discussed above.
For this, it is assumed that $p_c(y)=p(y;\xi_c)$, where $p(y;\xi)=\exp\Big\{\frac{y\xi-\psi(\xi)}{\phi}+c(y,\phi)\Big\}.$

\begin{defi}[Influence function of dual $\gamma$-mixture]
The influence function of the minimizer $p_{\gamma}(y;w)$ of the weighted sum of the dual $\gamma$-divergences $A_{\gamma}(q;w) = \sum_{c=1}^{C} w_c D_{\gamma}^{\ast}(q,p_c)$ for the outlier $\xi_{\out}$ is defined as
\begin{align}
    \mathrm{IF}(p_{\gamma}(y;w),\xi_{\out}) 
    =
    \lim_{\epsilon \to 0}
    \frac{ 
    p_{\gamma}(y;w(\epsilon)) - p_{\gamma}(y;w)
    }{
    \epsilon
    }.
\end{align}
\end{defi}

\begin{prop}
\label{prop:consIFgamma}
The influence function of the dual $\gamma$-mixture of exponential family distributions is bounded when $\gamma > 0$.
\end{prop}
\begin{proof}
Proof of this proposition is shown in Appendix~\ref{app:A4}.
\end{proof}

\section{Acquisition function with robust divergences}
In~\cite{10.5555/645527.757765}, the acquisition function is defined as 
\begin{align}
    a_0(\bx;w) =
    \sum_{c=1}^{C} w_c D_0(p(Y;\xi_c(\bx)),p(Y;\bar{\xi}(\bx))),
\end{align}
where the weight $w$ is explicitly denoted to consider the effect of an outlying committee member. 
Here, we also consider acquisition functions based on the $\beta$-divergence and the dual $\gamma$-power divergence given by
\begin{align}
    a_{\beta}(\bx;w) =
    \sum_{c=1}^{C} w_c D_{\beta}(p(Y;\xi_c(\bx)),p_{\beta}(Y|\bx))
\end{align}
and
\begin{align}
    a_{\gamma}(\bx;w) =
    \sum_{c=1}^{C} w_c D_{\gamma}^{\ast}(p(Y;\xi_c(\bx)),p_{\gamma}(Y|\bx)),
\end{align}
where $p_{\beta}$ and $p_{\gamma}$ are the consensus models obtained with respect to the $\beta$-divergence and the dual $\gamma$-power divergence, respectively.

Denoting the $\epsilon$-contaminated activation function as $a(\bm{x};w(\epsilon)) = $We define the influence function for the acquisition function $a(\bx;w)$ as follows.
\begin{defi}[Influence function of acquisition function]
The influence function of the acquisition function for the outlier $\xi_{\out}$ is defined as
\begin{align}
    \mathrm{IF}(a(\bx;w),\xi_{\out}) 
    =
    \lim_{\epsilon \to 0}
    \frac{ 
    a(\bx;w(\epsilon)) - a(\bx;w)
    }{
    \epsilon
    },
\end{align}
where $ a(\bx,w(\epsilon)) = W_{\epsilon}[D(p(y;\xi), p(y; W_{\epsilon}[\xi]))].$
\end{defi}

Then we have the following proposition.
\begin{prop}
\label{prop:IFac}
The influence function for $a_0$ is not bounded, whereas those for $a_{\beta}$ and $a_{\gamma}$ are bounded with respect to an outlier $\xi_{\out}$.
\end{prop}
\begin{proof}
Proof of this proposition is shown in Appendix~\ref{app:A5}.
\end{proof}
This proposition claims that, as is the case for the consensus model, the mixtures based on the KL divergence is vulnerable to outlier in the committee while those based on the $\beta$ and the dual $\gamma$-power divergence are robust.

\section{Experiments}
Logistic regression is used as the predictive model. The model is fitted using the initial training dataset $S_0$. Then, 10 logistic regression models are trained on 10 partitions of the labeled data at hand and used as committee members, where $w_c = 1/10, \; c=1,\dots, 10$. The consensus models are constructed on the basis of the KL divergence and the $\beta$ divergence with $\beta=1.0$ and the dual $\gamma$-power divergence with $\gamma = 1.0$, and by using the acquisition function using them, we select one datum from the pool dataset $\cX_p$. The correct label is assigned to the selected sample and added to the training data, and the predictive model is retrained using the extended training dataset. As a baseline, we also compare the results with those obtained by replacing the selection by acquisition function with random sampling. 

The following one artificial dataset and six real-world datasets from the LIBSVM datasets\footnote{\url{https://www.csie.ntu.edu.tw/~cjlin/libsvmtools/datasets/}} are considered.
\begin{enumerate}
    \item {\tt{artificial}}: It is an artificially generated three-dimensional dataset for two-class classification. The samples are drawn from $\mathcal{N}(\bm{\mu}_i,\Sigma), i=0,1$, where $\bm{\mu}_0 = (1,1,1)$ and $\bm{\mu}_1 = - \bm{\mu}_0$, and $\Sigma = I_3$ (unit matrix). 
    The pooled dataset is composed of $1,000$ data points for each class. 
    The prediction error is evaluated by using $50,000$ data points for each class generated in the same manner as the training dataset.
    By three sampling methods, we sequentially selected $100$ samples.
    \item {\tt{adult}}: The original number of attributes is 123 and is reduced to three by PCA. 100 samples are sequentially selected by three sampling methods.
    The original sample has 48,842 data points with 37,155 positives and 11,687 negatives. 
    \item {\tt{breast-cancer}}: The original number of attributes is 9 and is reduced to three by PCA. 100 samples are sequentially selected by three sampling methods.
    The original sample has 4,000 data points with 2,839 positives and 1,161 negatives. 
    \item {\tt{diabetis}}: The original number of attributes is 8 and is reduced to three by PCA. 100 samples are sequentially selected by three sampling methods.
    The original sample has 9,360 data points with 6,078 positives and 3,282 negatives. 
    \item {\tt{mushrooms}}: The original number of attributes is 112, and is reduced to three by PCA. 100 samples are sequentially selected by three sampling methods. 
    The original sample has 8,124 data points with 4,208 positives and 3,916 negatives. 
    \item {\tt{ijcnn}}: The original number of attributes is 22, and is reduced to three by PCA. 100 samples are sequentially selected by three sampling methods. The original sample has 35,000 data points with 31,585 positives and 3,415 negatives. 
    \item {\tt{titanic}}: The original number of attributes is three. 100 samples are sequentially selected by three sampling methods.
    The original sample has 3,000 data points with 2,009 positives and 991 negatives. 
\end{enumerate}

The initial number of data points is 50 for each class. For the three real-world datasets, the pool and test datasets are generated as follows. 
Keeping the initial training dataset with 100 data points, let the number of data points in the minority class be $m$. As the pool dataset, $m$ data points are randomly drawn from each of the positive and negative datasets, and the remaining data points are used for evaluating the prediction error. 
Averages and standard deviations of 10 random samplings of initial data, random sampling, and random splitting for learning committee members are shown in Figure~\ref{fig:results}. 
\begin{figure*}[t!]
\centering
\includegraphics[scale=0.26]{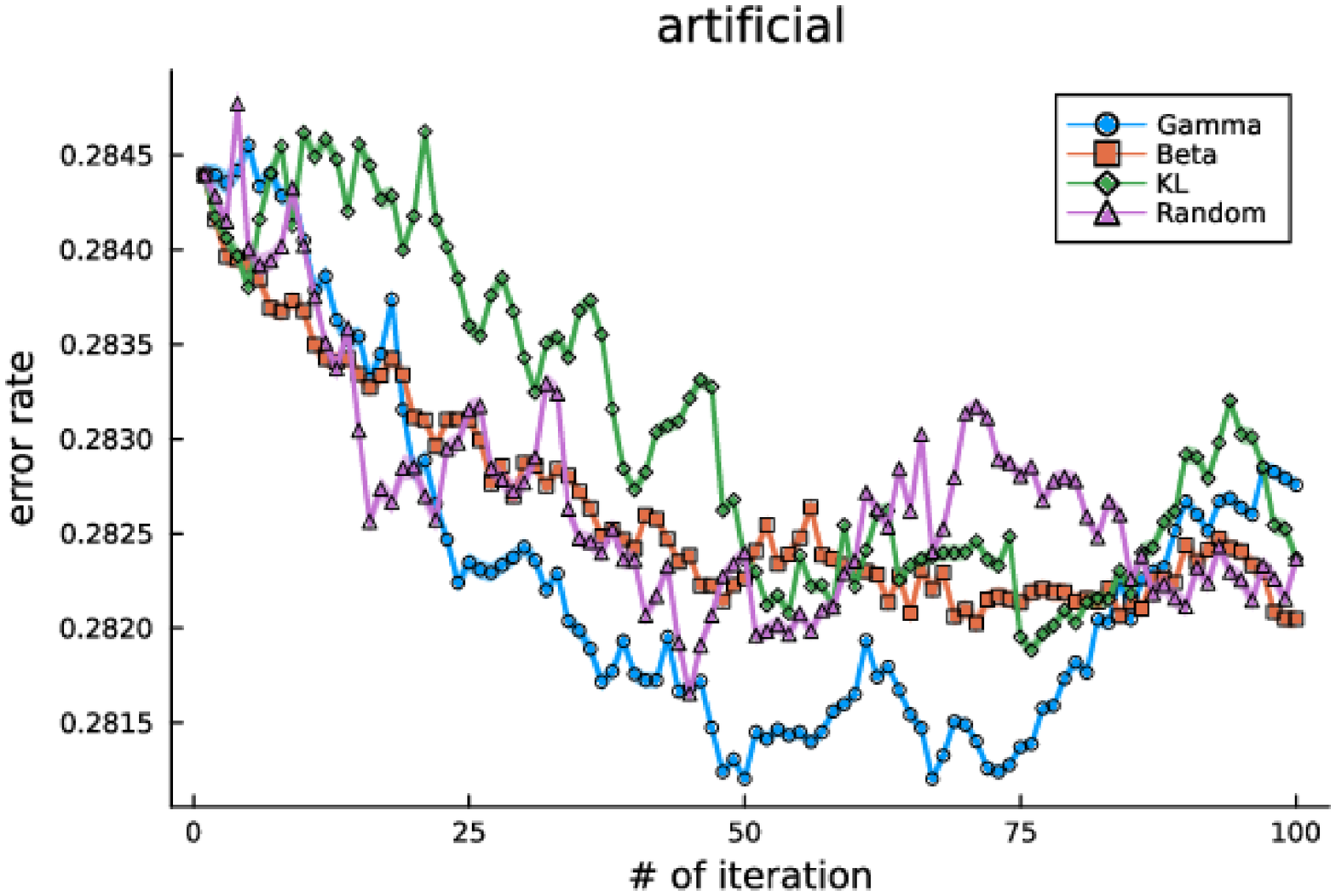}
\includegraphics[scale=0.26]{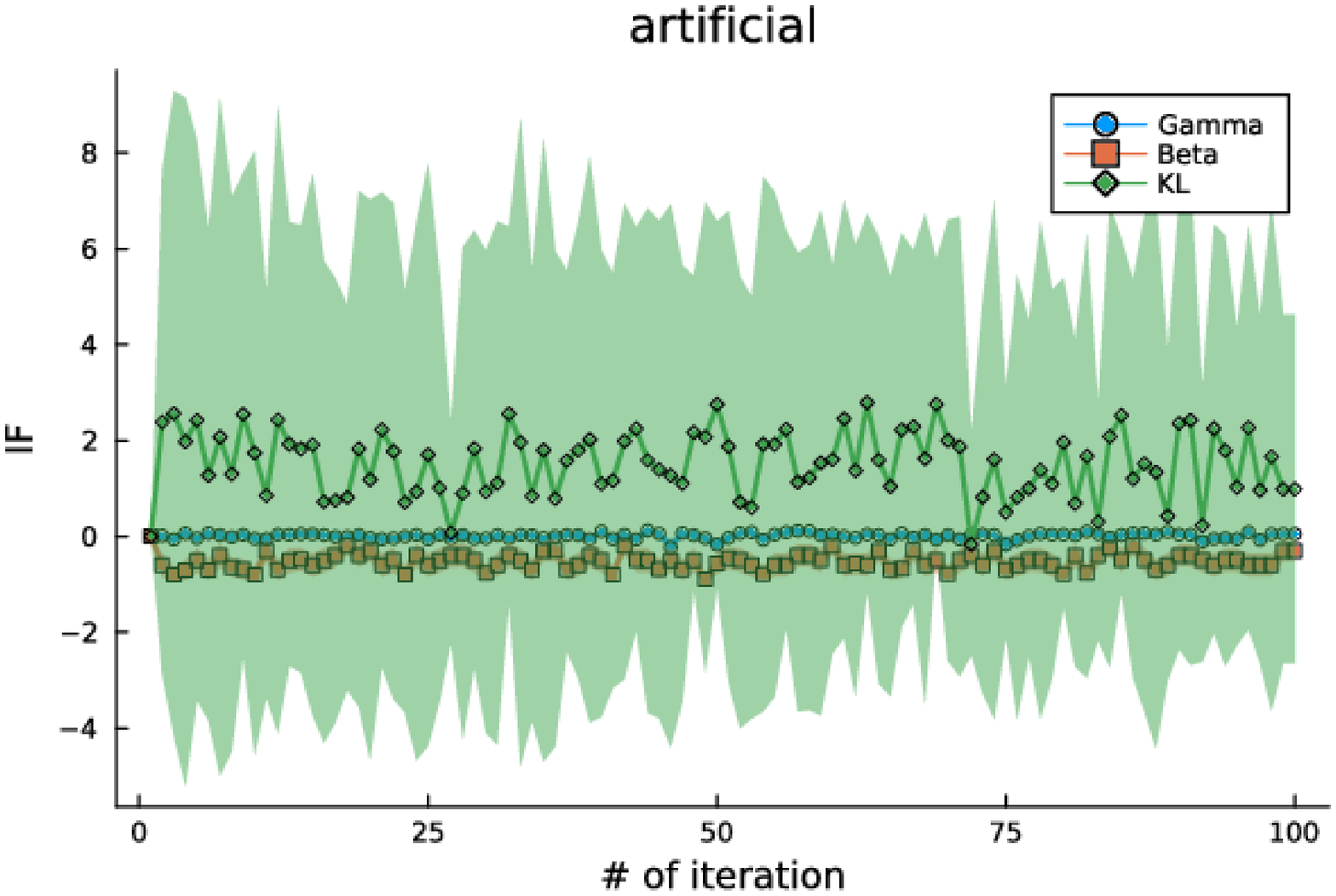}\\
\includegraphics[scale=0.26]{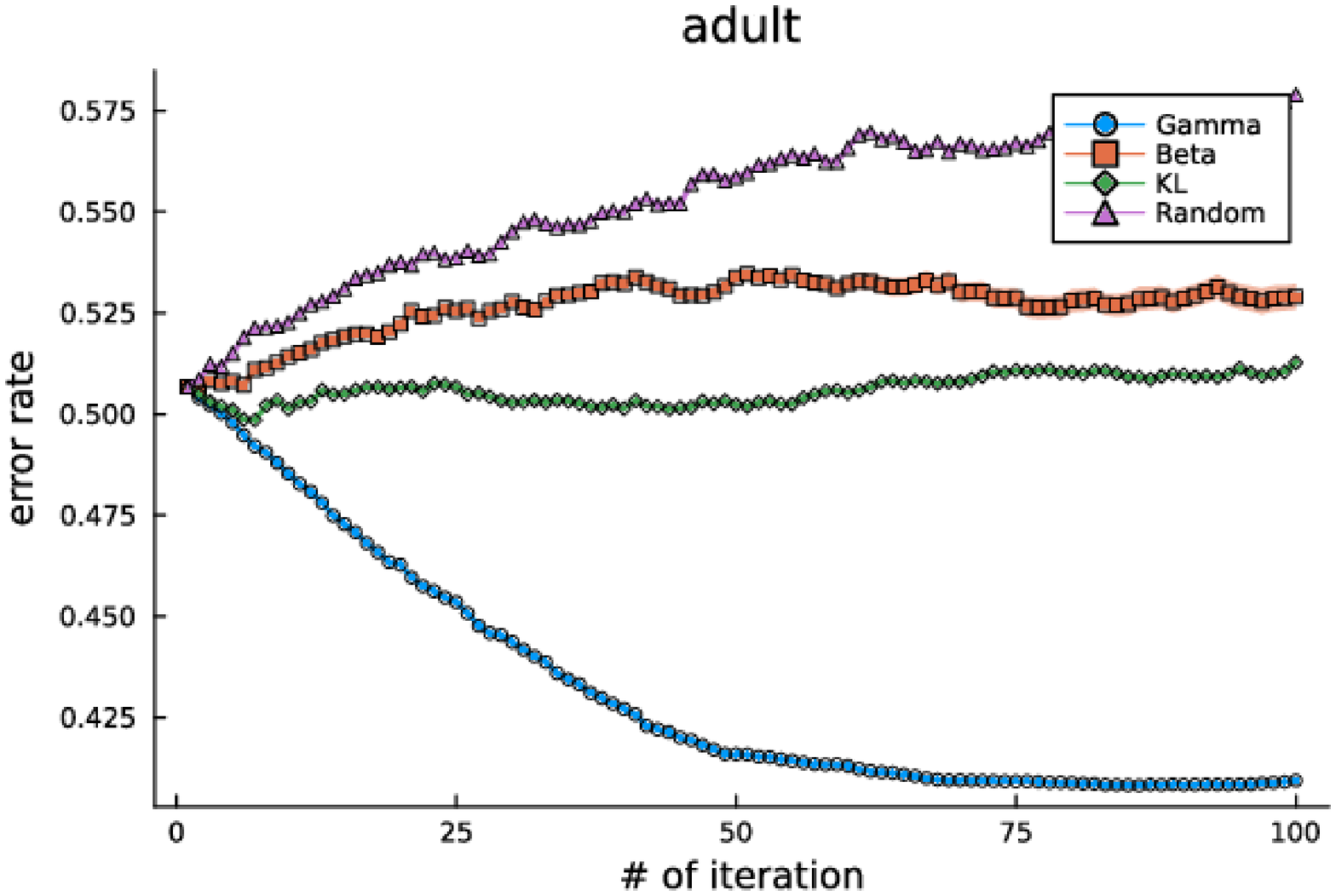}
\includegraphics[scale=0.26]{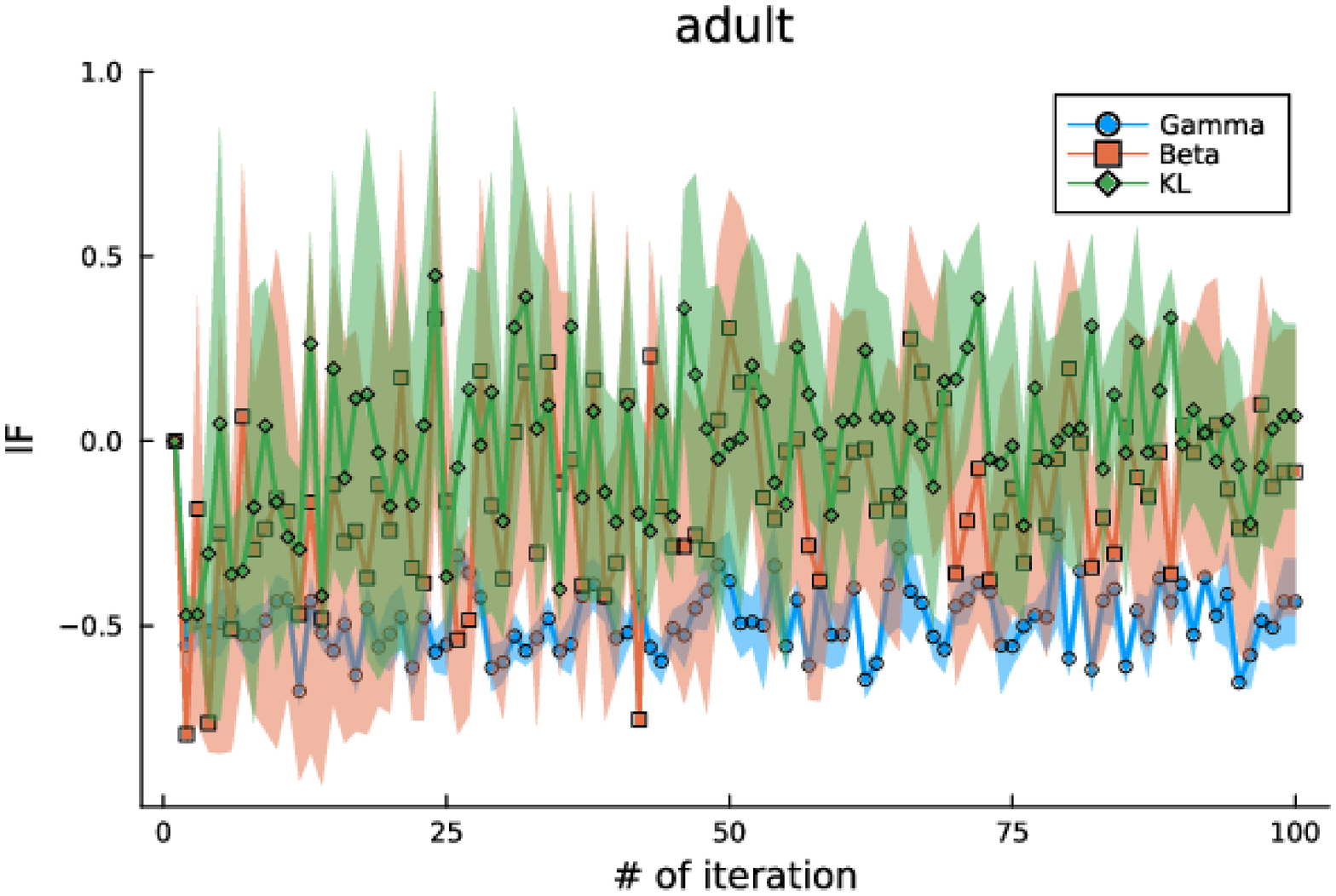}\\
\includegraphics[scale=0.26]{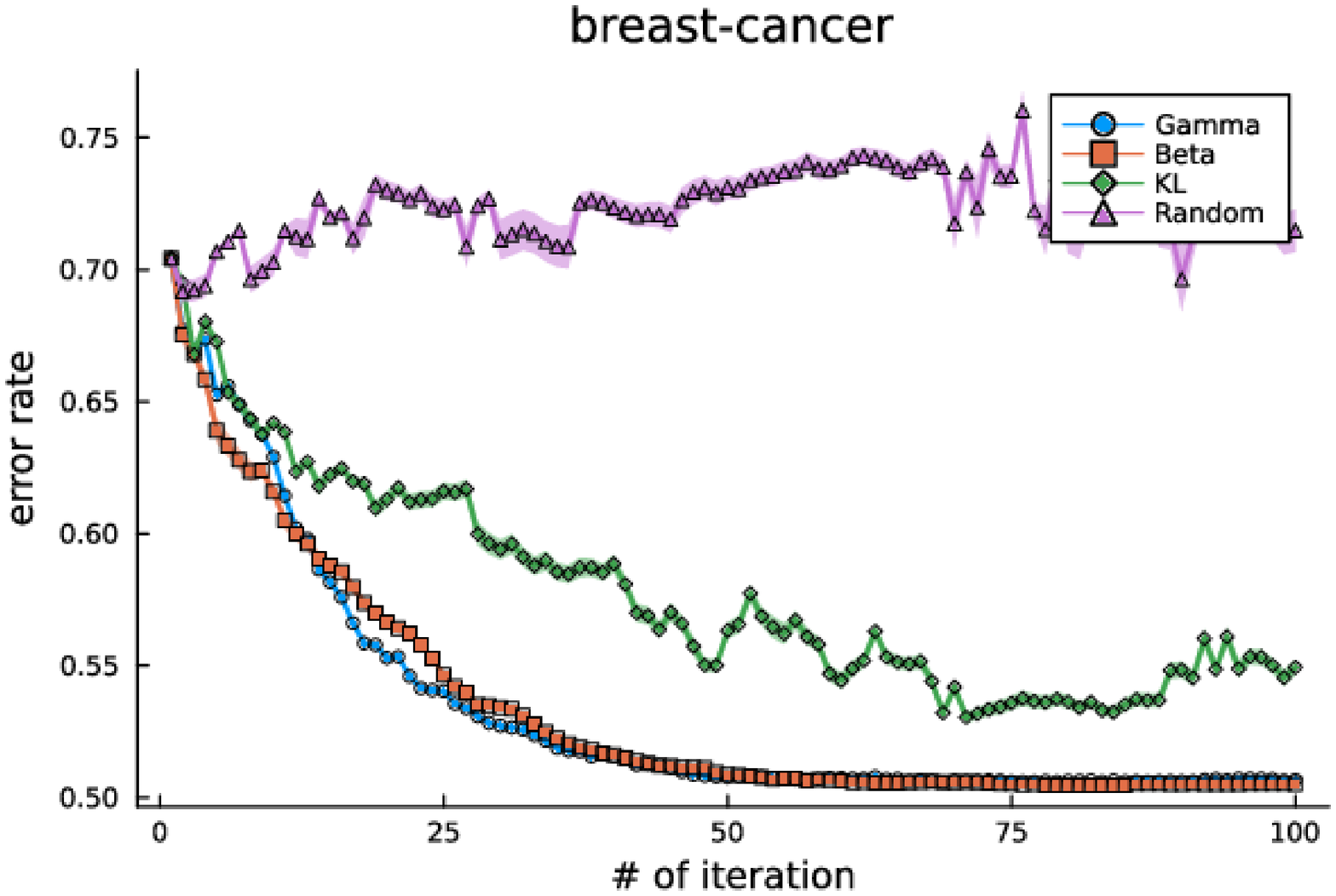}
\includegraphics[scale=0.26]{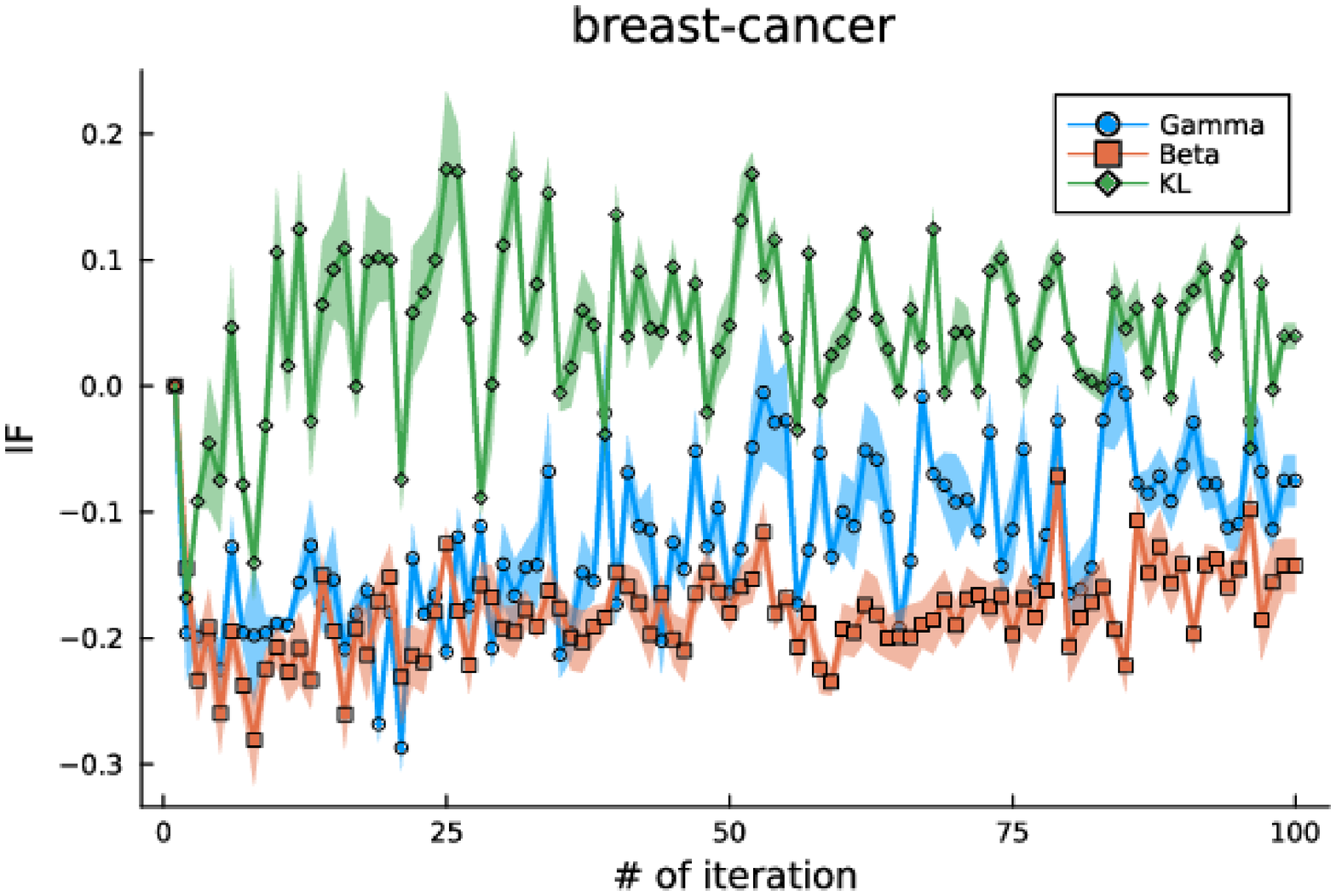}\\
\includegraphics[scale=0.26]{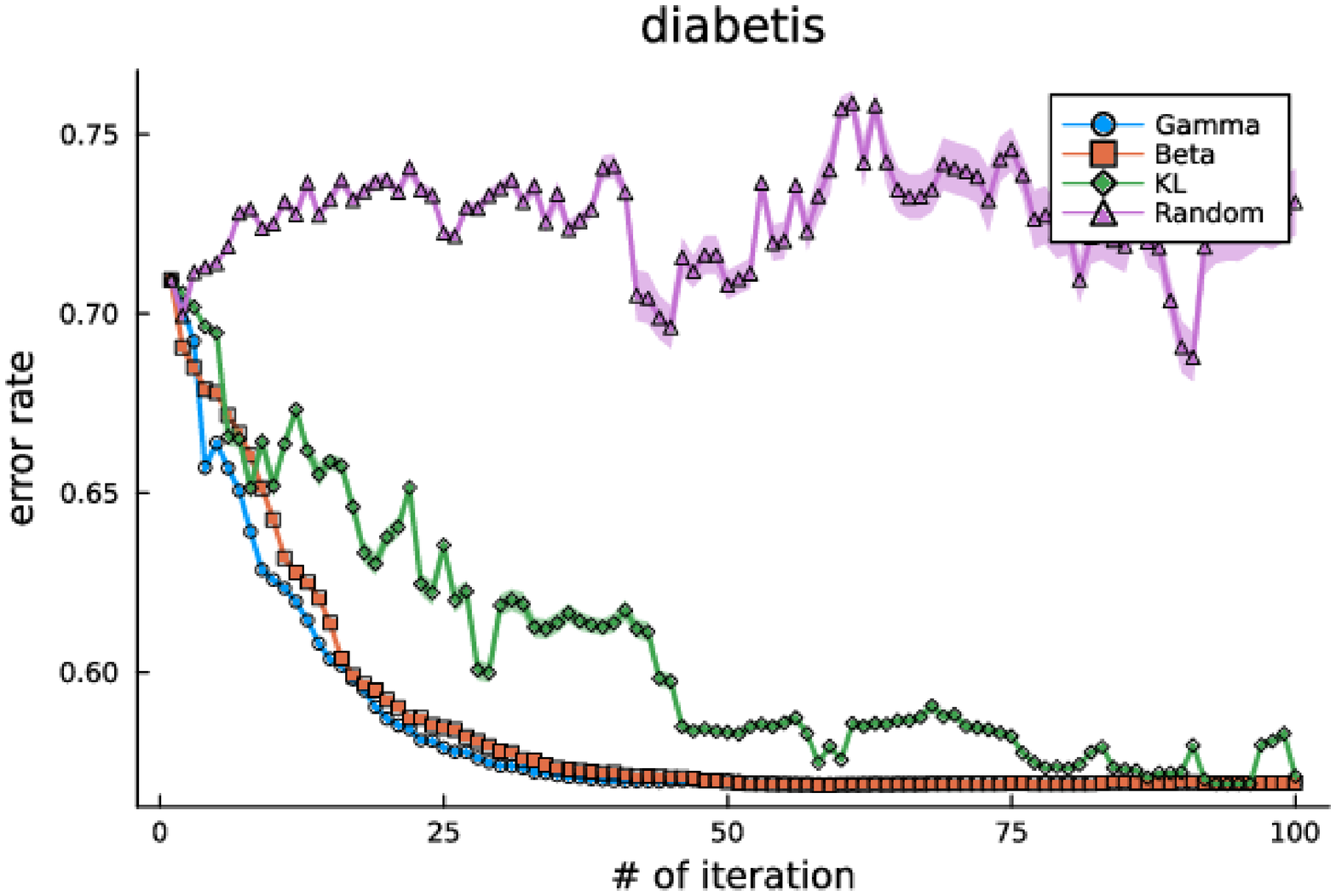}
\includegraphics[scale=0.26]{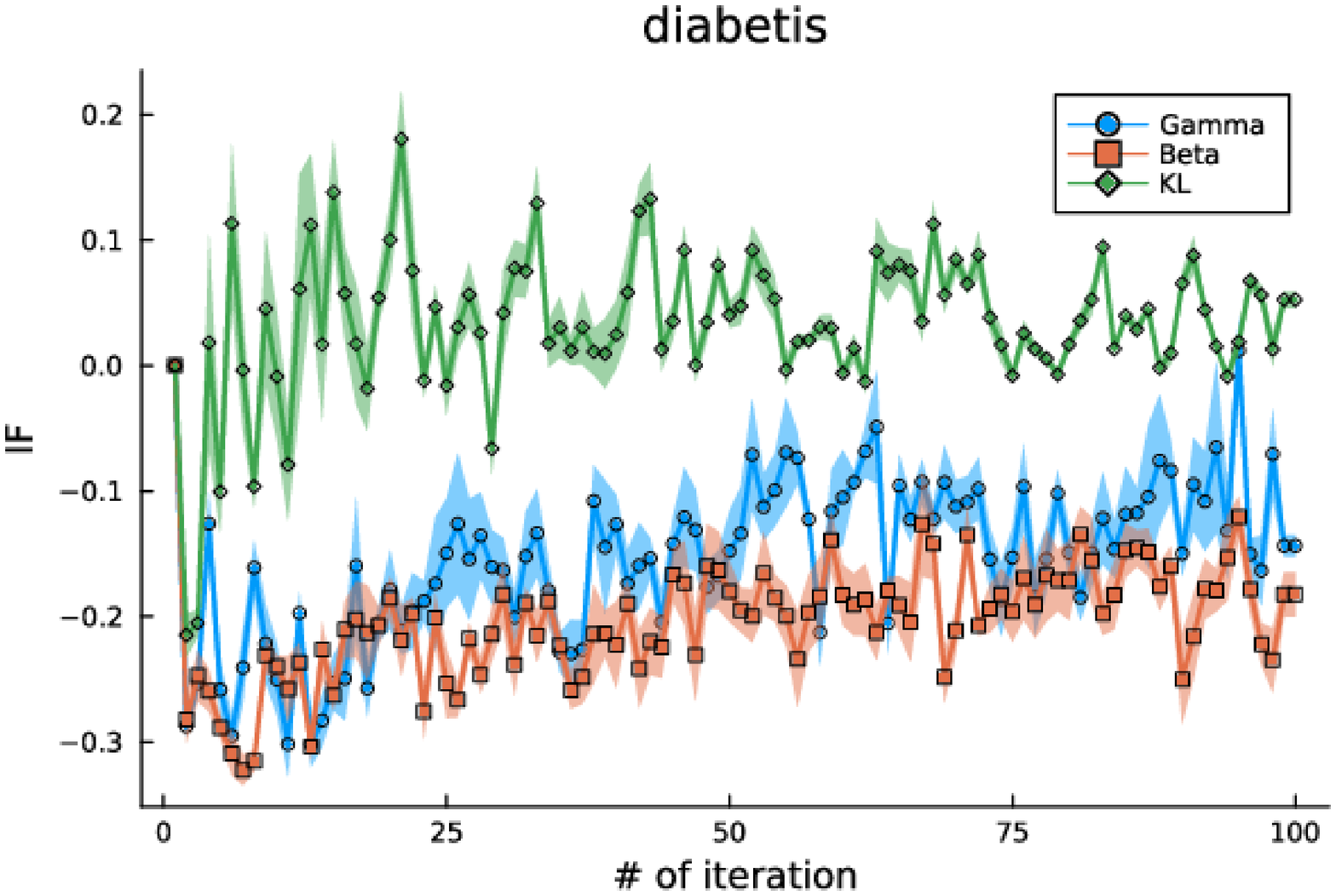}
\caption{Left: Prediction error and the number of acquired data points in four sampling methods for dataset {\tt{artificial}}, {\tt{adult}}, {\tt{breast-cancer}}, {\tt{diabetis}}. Right: Influence function values at queried points with $\xi_{\out}$ being the absolute maximum in the pool.}
\label{fig:results}
\end{figure*}
\begin{figure*}[t!]
\centering
\includegraphics[scale=0.26]{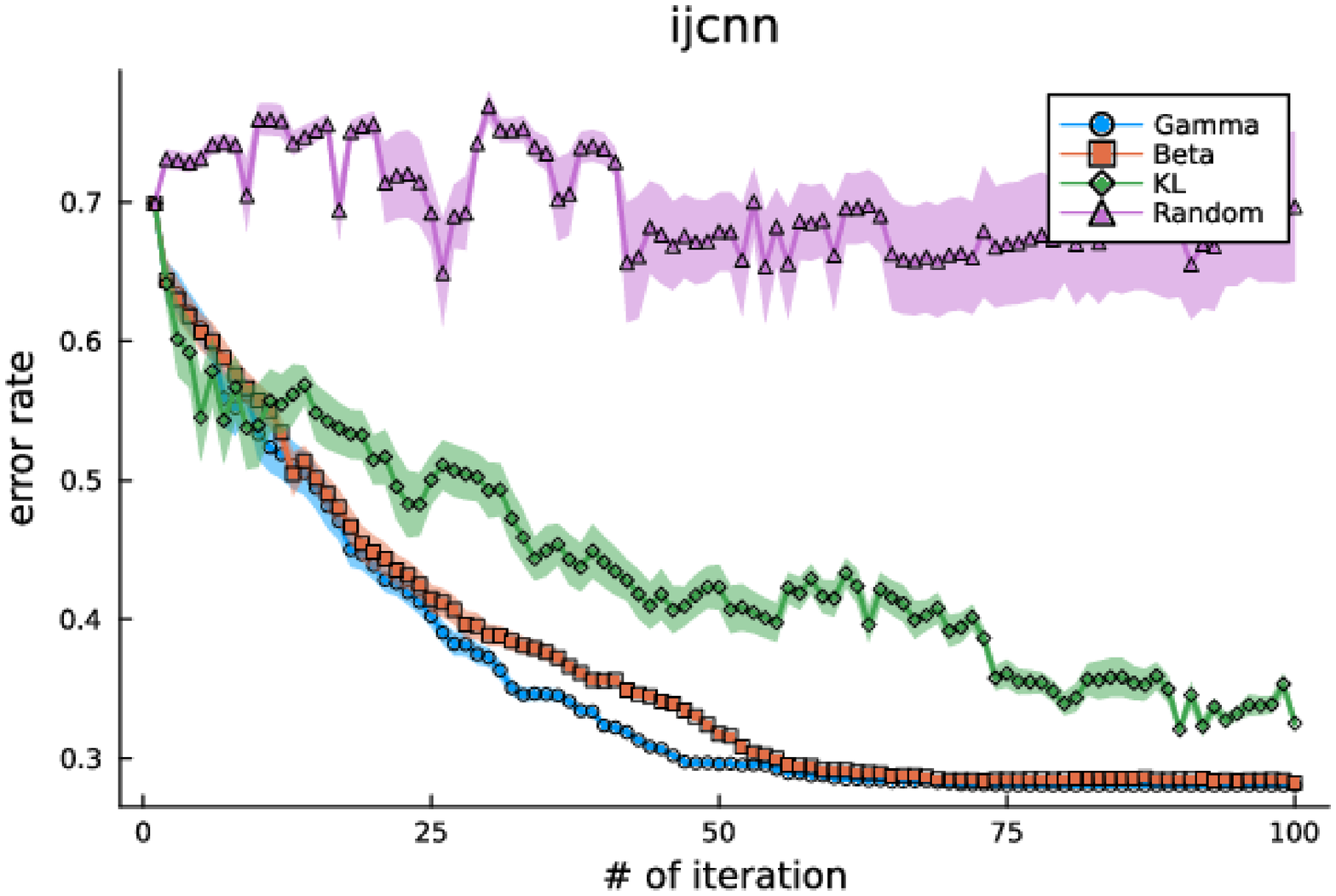}
\includegraphics[scale=0.26]{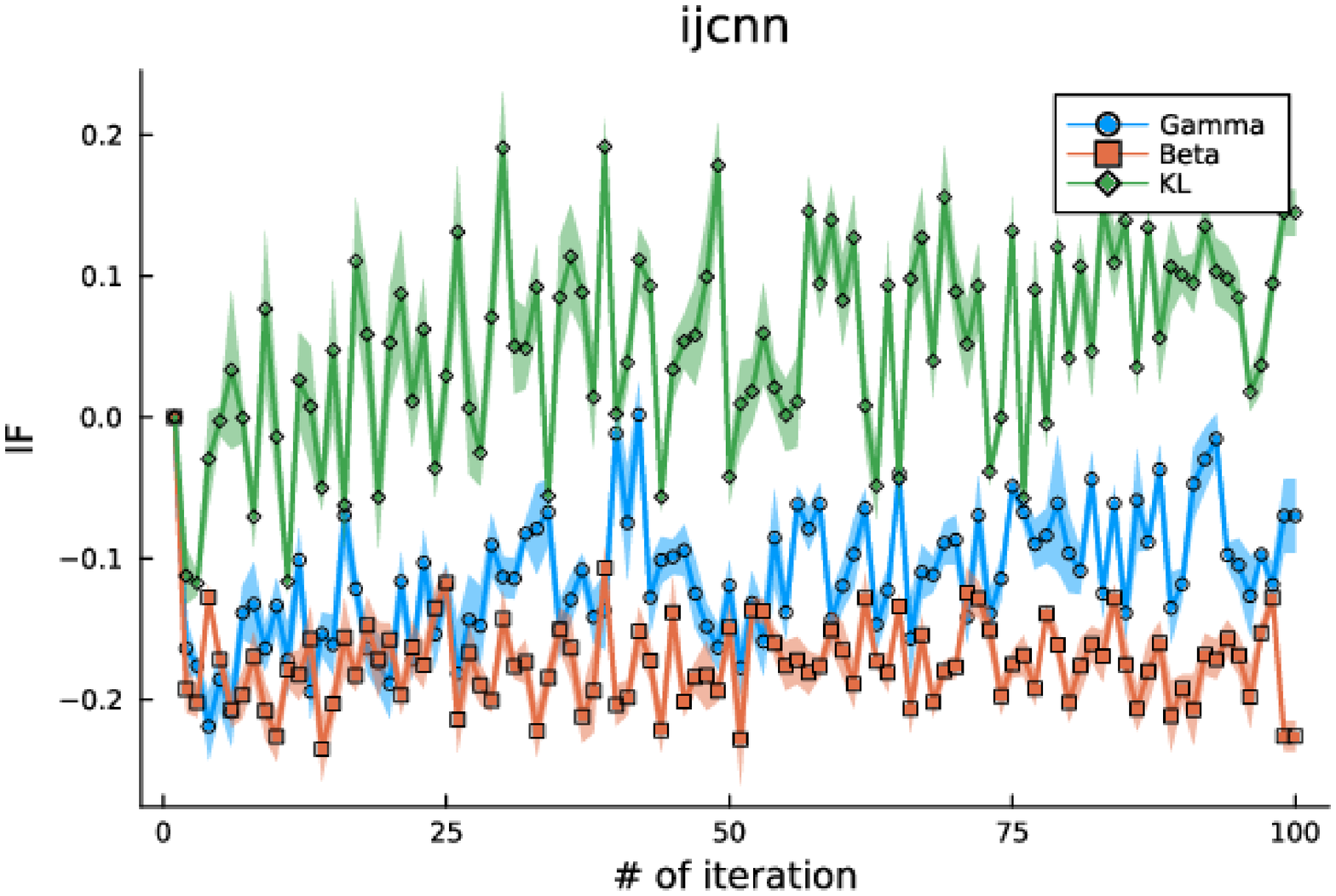}\\
\includegraphics[scale=0.26]{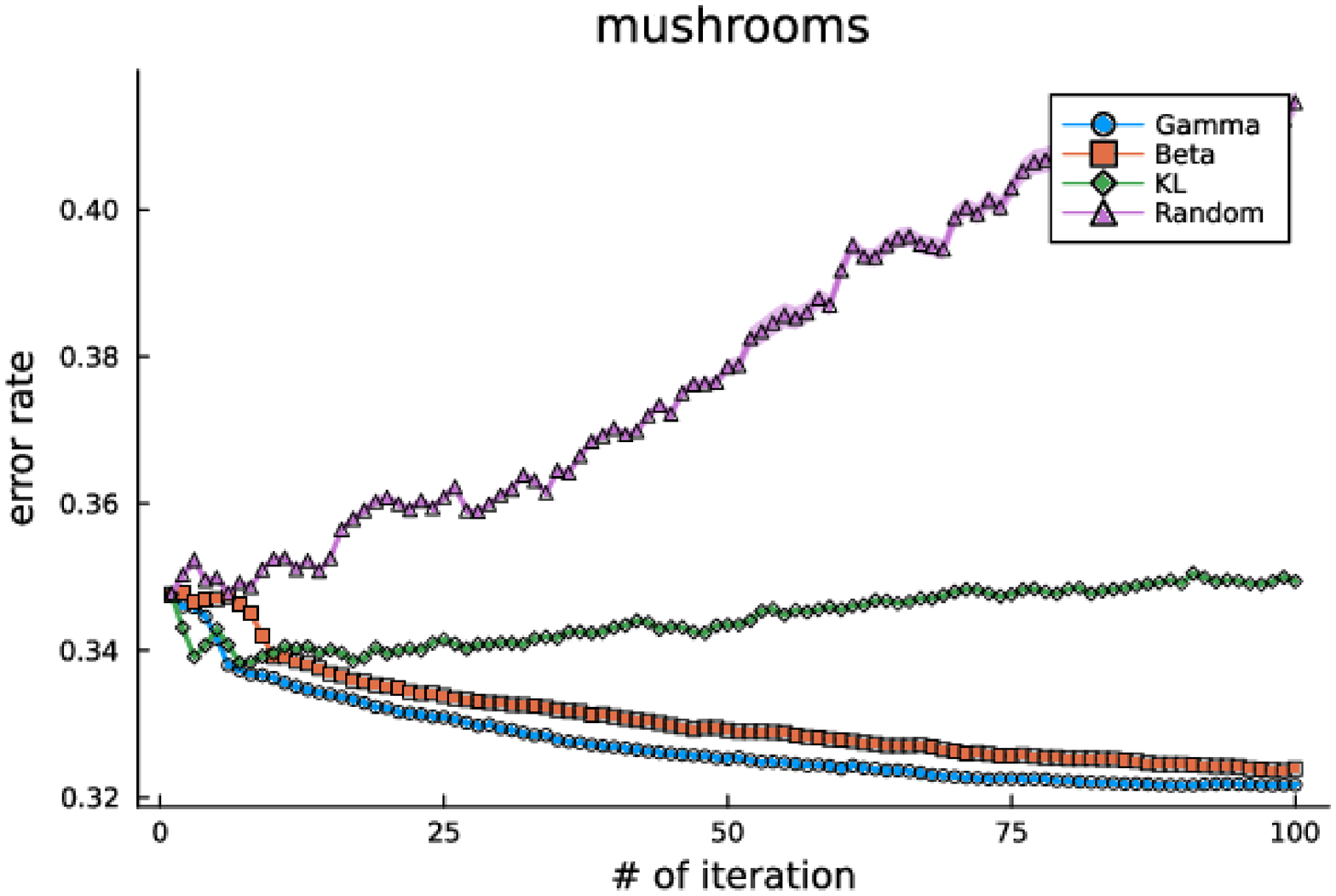}
\includegraphics[scale=0.26]{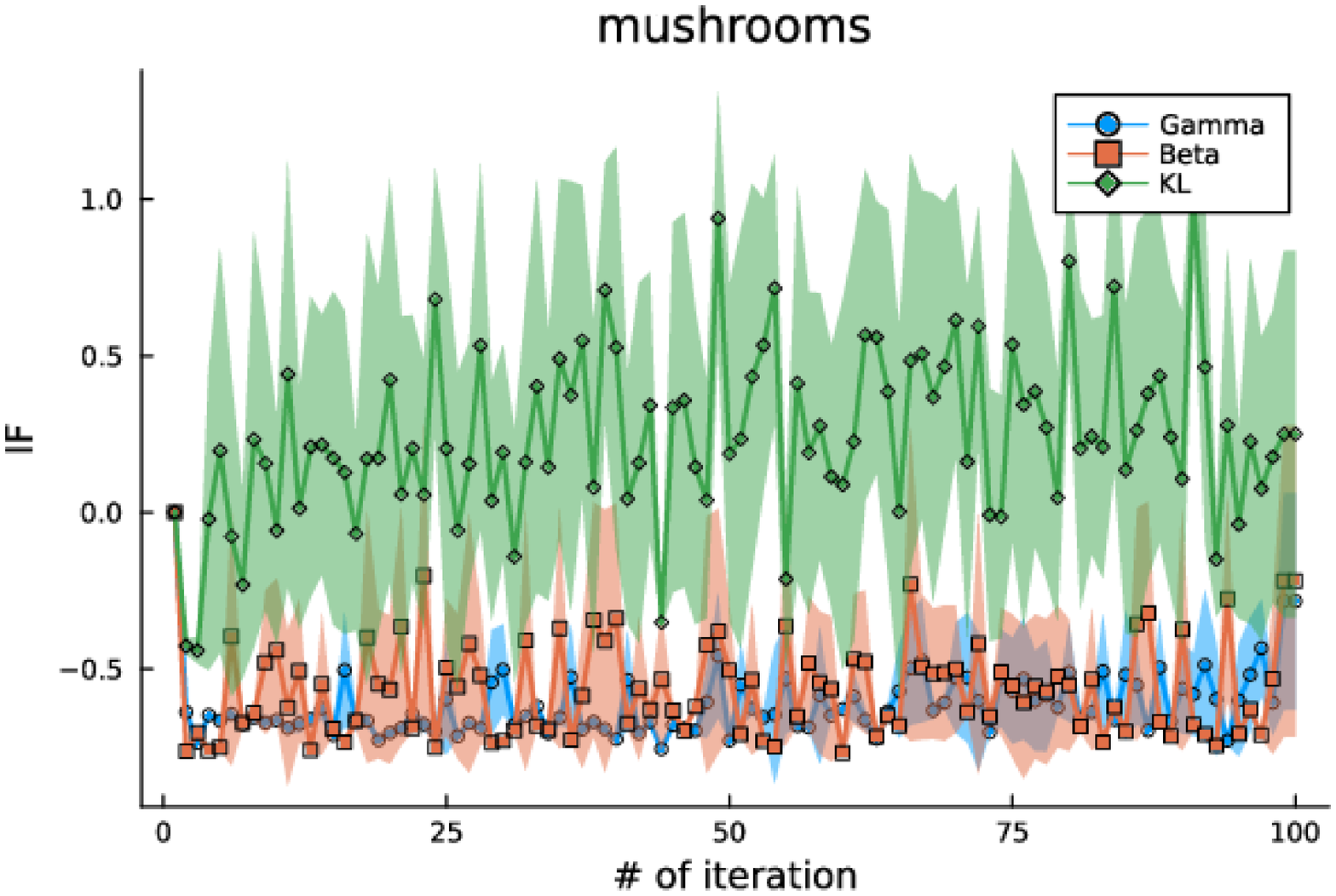}\\
\includegraphics[scale=0.26]{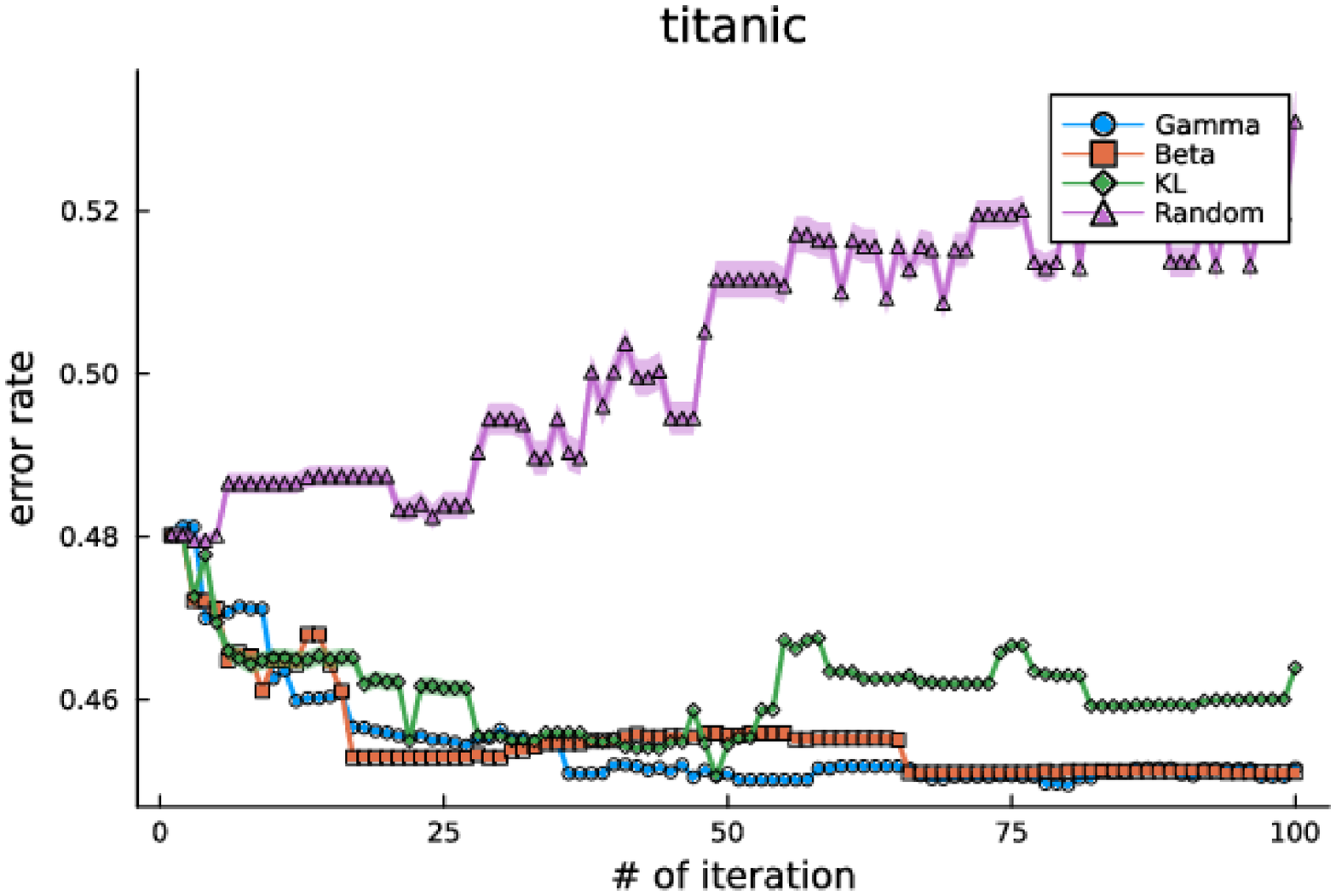}
\includegraphics[scale=0.26]{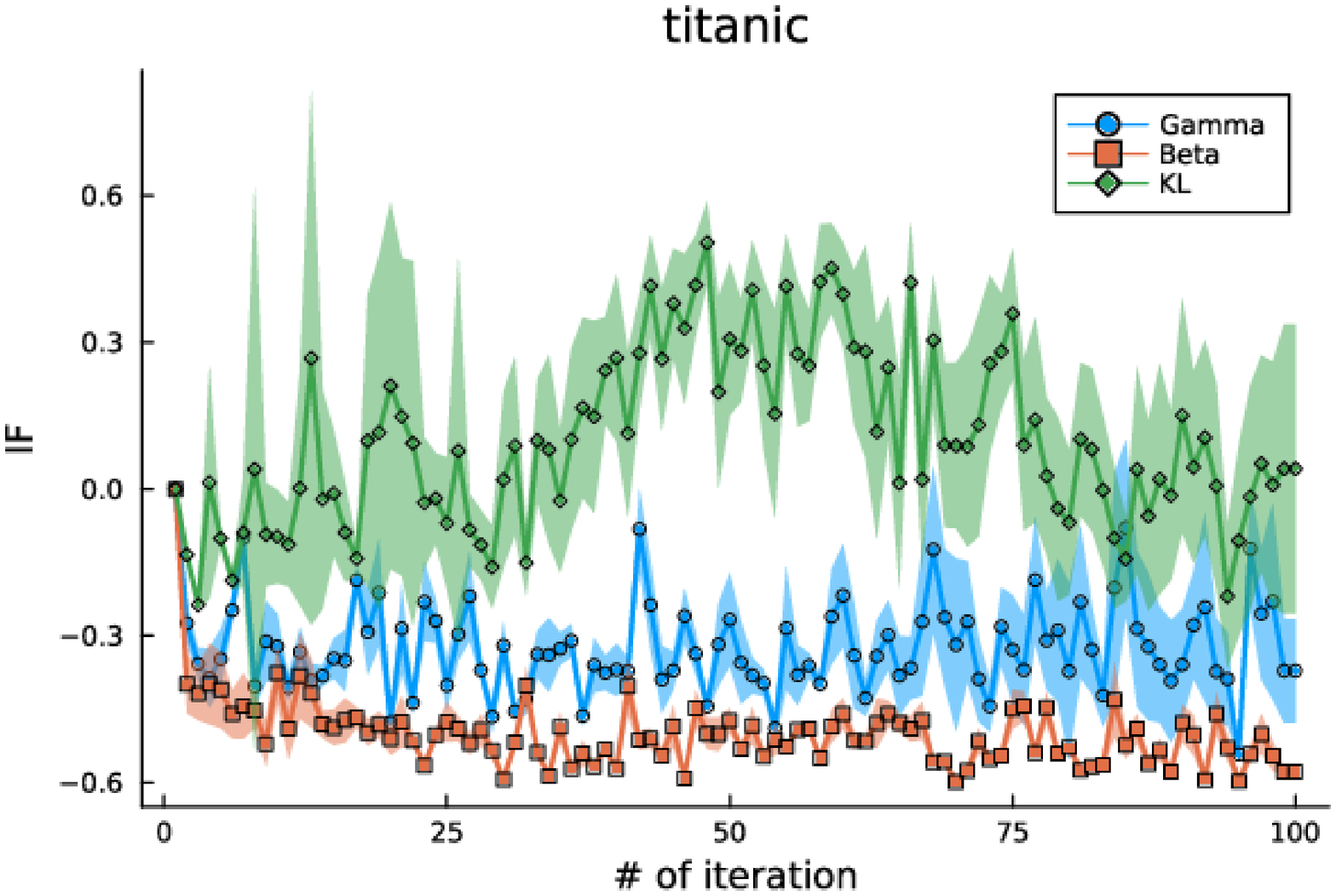}
\caption{Left: Prediction error and the number of acquired data points in four sampling methods for dataset {\tt{ijcnn}}, {\tt{mushrooms}}, {\tt{titanic}}. Right: Influence function values at queried points with $\xi_{\mathrm{out}}$ being the absolute maximum in the pool.}
\label{fig:results2}
\end{figure*}
From left panels of Figs.~\ref{fig:results} and~\ref{fig:results2}, it is seen that the consensus model and the acquisition function derived from the $\beta$ divergence and the dual $\gamma$-power divergence provide smaller prediction errors in many cases. The conventional method with the KL divergence is not stable in the early stage of active learning, whereas the proposed method enables the selection of better samples to measure, which should be attributed to its robustness to a poor committee member in the early stage of active learning. 

We also show the values of the influence function at queried points in the right panels of Figs~\ref{fig:results} and~\ref{fig:results2}. As an outlying point, we find an input point $\bx \in \mathcal{X}_{p}$ from the pool which maximizes the absolute value of $\xi = \lan \bt,\bx \ran$ where $\bt$ is the regression coefficient of the current predictive model. From these figures, it is clearly seen that the value of the influence function for the method based on the KL-divergence is higher than the other two methods, which is consistent with the theoretical result. It is also seen that for the {\tt{adult}} data, the method based on the $\beta$ divergence does not perform well in terms of error rate, and the influence function values are as high as those of the KL divergence-based method.

\section{Conclusion and Discussion}
In this paper, we revisited a classical active learning method QBC. By replacing the KL-divergence with the $\beta$-divergence or the dual $\gamma$-power divergence, we were able to obtain a favorable performance both theoretically and experimentally.

The $u$-mixture for distributions in an exponential family does not have a closed form solution in general; hence, we used an implicit characterization of the $u$-mixture proved in~\cite{muratafujimoto2009}. We also considered the dual $\gamma$-power mixture based on a scale-invariant divergence. It is proven that it has a similar form to the $u$-mixture, but it has an advantage that the normalization factor is always computable; hence, the dual $\gamma$-power mixture is always defined unlike the $u$-mixture. 
When we fix a divergence measure for probability distributions and impose a constant mean condition, we obtain a class of the maximum entropy distributions~\cite{KE2022}. For the KL-divergence, the exponential family is derived as a class of maximum entropy distributions. On the other hand, when we adopt the $u$-divergence, the associated maximum $u$-entropy distribution is of a different form compared with the case of the exponential family, and the consensus model derived using the $u$-mixture of such a class of distributions is easy to obtain by the arithmetic average of the model parameters, as in the case of the consensus model parameter $\bar{\bt}$ for the conventional QBC. In other words, the statistical model and the estimation procedure have a dualistic structure in combination, and the use of the $u$-mixture of the exponential family distributions is in this sense an unnatural procedure. However, this inconsistency of the model and estimation is a source of robustness~\cite{10.1016/j.jmva.2008.02.004}. In our future work, we will explore a more detailed geometric characterization of active learning based on, for example, Pythagoras foliation associated with the Bregman divergence.

The problem of model selection, namely, the appropriate choice of $\beta$ or $\gamma$ parameter, is an interesting open problem. In the literature on robust regression, the optimization of the parameter $\gamma$ for the $\gamma$-power divergence is considered in~\cite{e22040399} via the notions of asymptotic efficiency and breakdown point using the theory of S-estimation. Methods to select an appropriate parameter for robust divergence used in active learning would be explored. 
Another important issue is the adaptive selection of divergence measures. After a sufficient sample has been collected, the KL divergence-based method will work well. Even if we use robust divergence-based methods, it may fail, for example in the case of $\beta$ divergence-based method on adult data. Furthermore, since robust divergence has tuning parameters, its adjustment also affects the performance of active learning. One possible strategy is to collect samples using highly robust methods and parameters in the early stages of active learning. Then, change the divergence measure with the KL divergence-based ones at a later stage. In this case, criteria regarding the number of samples to be collected and divergence selection are necessary and should be considered in conjunction with the parameter selection issue.

\section*{Acknowledgments}
Part of this work is supported by JSPS KAKENHI No. JP18H03211, JP22H03653 and NEDO JPNP18002 and JST CREST No.JPMJCR2015. 

\section*{Data Availability}
Source code is available upon reasonable request.

\appendix 
\section{Proofs}
\subsection{Proof of Theorem~\ref{theo:umix}}
\label{app:A}
For the sake of notational simplicity, we introduce a notation for the norm of a function $f$ as $\lan f \ran = \int f(y) \dd \Lambda (y)$. Then, we have
\begin{align}
\notag
A_{U}(q;w)=&
\sum_c w_c \{ \lan U(\breve{p}_c) \ran - 
\lan U ( \breve{q}) \ran 
-
\lan q, \breve{p}_c - \breve{q} \ran \} \\
\notag
=&
\sum_c w_c \lan U ( \breve{p}_c)\ran - 
\lan U(\breve{q})\ran - 
\lan q , \sum_c w_c \breve{p}_c - \breve{q} \ran \\
\notag
=&
\lan U ( \breve{p}_u ) \ran - \lan U(\breve{q}) \ran 
- \lan q, \breve{p}_u - \breve{q} \ran \\
\notag
&- 
\lan U(\breve{p}_u)\ran + \sum_c w_c \lan U(\breve{p}_c) \ran - b\lan q\ran \\
=&
D_U(q,p_u) - b - \lan U(\breve{p}_u) \ran + \sum_c w_c \lan U(\breve{p}_c)\ran.
\end{align}
In the last equation, only the first term depends on $q$; hence, $q=p_u$ minimizes the weighted sum of the Bregman divergences.

\subsection{Proof of proposition~\ref{prop:consIFbeta}}
\label{app:B}
Consider the KL-divergence with $u(z) = \exp u$ and $u^{\ast}(\zeta) = \log \zeta$. Then, the consensus model is given by the model $p(y;\bar{\xi},w) $ with the parameter $\bar{\xi}= \sum_c w_c \xi_c$. 
Denoting the action of the $\epsilon$-contamination $w_\epsilon$ to a function $f(\xi)$ as 
\begin{align}
    W_{\epsilon}[f] = \int  ((1-\epsilon) w_c \delta(\xi - \xi_c) + \epsilon\delta(\xi - \xi_{\out})) f(\xi) \dd \xi =
    (1-\epsilon) \sum_{c=1}^{C} w_c f(\xi_c) + \epsilon f(\xi_{\out}),
\end{align}
we have
\begin{align}
p(y;\bar{\xi},w(\epsilon)) =&
\exp 
\left( 
\phi^{-1}
 (y W_{\epsilon}[\xi] - \psi (W_{\epsilon}[\xi]))
+c(y,\phi)
\right).
\end{align}
Taking the derivative with respect to $\epsilon$, we have the influence function as 
\begin{align}
    \mathrm{IF}(p(y;\bar{\xi},w),\xi_{\out}) =&
    p(y;\bar{\xi},w)
    \left\{ 
    \frac{y \xi_{\out} - \psi^{\prime}(\xi_{\out})}{\phi} 
    -
    \frac{y \bar{\xi} - \sum_c w_c \psi^{\prime}(\xi_c)}{\phi}
    \right\}.
\end{align}
Hence, the influence function is unbounded for $\xi_{\out}$. On the other hand, consider the $\beta$-divergence with $u(z) = (\beta z +1)^{1/\beta}$ and $u^{\ast}(\zeta) = \frac{\zeta^{\beta} - 1}{\beta}$. Then, the minimizer $p_u(y;w(\epsilon))$ of $A_U(q;w(\epsilon))$ is 
\begin{align}
     p_u(y;w(\epsilon)) = &
     \left[
     W_{\epsilon}[
     \exp 
     \left\{ 
     \phi^{-1} \beta ( y \xi_c - \psi(\xi_c)) + \beta c(y,\phi)
     \right\}
     ]
     -\beta b(w(\epsilon))
     \right]^{1/\beta}.
\end{align}
Taking the derivative with respect to $\epsilon$, we have 
the influence function as
\begin{align} \notag
    \mathrm{IF}(p_u(y;w);\xi_{\out}) =&
\frac{1}{\beta p_u^{1-\frac{1}{\beta}}}
\left[
\int \{ \delta(\xi - \xi_{\out}) - \sum_{c} w_c \delta(\xi - \xi_c)\}
\exp 
\{
\phi^{-1} \beta (y \xi - \psi(\xi)) + \beta c(y,\phi)
\}\dd \xi -\beta b
\right]\\
=&
\frac{1}{\beta p_u^{1-\frac{1}{\beta}}} \{ p_{\out}^\beta - \overline{p^{\beta}}\},
\end{align}
where we used simplified notations $p_c = p(y;\xi_c)$, $p_{\out} = p(y|\xi_{\out})$ and 
\begin{align}
\overline{p^{\beta}} = \sum_c w_c p_c^{\beta}. 
\end{align}
We conclude that if $\beta >0$, then there exists a limit of $ \mathrm{IF}(p_u(y;w),\xi_{\out})$ when $|\xi_{\out}| \to \infty$.

\subsection{Proof of proposition~\ref{prop:gammamix}}
\label{app:A3}
A direct observation gives that
\begin{align} \notag
A_\gamma (q)-A_\gamma (p_\gamma (\cdot;w)) =&
- \frac{\int q\sum_c w_cp_c^\gamma{\rm d}\Lambda}{(\int q^{\gamma+1}{\rm d}\Lambda)^{\frac{1}{\gamma+1}}}
+ \frac{\int p_\gamma (\cdot;w) \sum_c w_c p_c^\gamma{\rm d}\Lambda}{(\int 
p_\gamma(\cdot;w)^{\gamma+1}{\rm d}\Lambda)^{\frac{1}{\gamma+1}}}\\ \notag
=&- \frac{z(w)^\gamma\int q\>p_\gamma(\cdot;w)^{\gamma}{\rm d}\Lambda}{(\int q^{\gamma+1}{\rm d}\Lambda)^{\frac{1}{\gamma+1}}}
+z(w)^\gamma\bigg(\int p_\gamma(\cdot;w)^{\gamma+1}{\rm d}\Lambda\bigg)^{\frac{\gamma}{\gamma+1}}
\end{align}
which is equal to $z(w)^\gamma D_\gamma(q,p_\gamma(\cdot;w))$.
This leads to the conclusion that $A_\gamma(q)\geq A_\gamma(p_\gamma(\cdot;w)) $, and the equality
holds if and only if $q=p_\gamma(\cdot;w)$.

\subsection{Proof of proposition~\ref{prop:consIFgamma}}
\label{app:A4}
The consensus model with the dual $\gamma$-power divergence is given by
\begin{align}
    p_{\gamma}(y;w) =& \frac{1}{z(w)} \left( \sum_c w_c p_c(y)^{\gamma}\right)^{\frac{1}{\gamma}} =
    \frac{1}{z(w)} \left( \overline{p(y)^{\gamma}}  \right)^{\frac{1}{\gamma}},\\
    z(w) =& \int 
    \left(
    \overline{p(y)^{\gamma}}
    \right)^{\frac{1}{\gamma}} \dd y.
\end{align}
The $\epsilon$-contamination model is 
\begin{align}
     p_{\gamma}(y;w(\epsilon)) =&
\frac{ 
\left(W_{\epsilon}[p^{\gamma}(y;\xi)]\right)^{\frac{1}{\gamma}}
}
{
\int
\left(W_{\epsilon}[p^{\gamma}(y;\xi)]\right)^{\frac{1}{\gamma}}
}
\dd \Lambda(y).
\end{align}
The derivative of $\left(W_{\epsilon}[p^{\gamma}(y;\xi)]\right)^{\frac{1}{\gamma}}$ is 
\begin{align}
    \frac{\partial \left(W_{\epsilon}[p^{\gamma}(y;\xi)]\right)^{\frac{1}{\gamma}}}{\partial \epsilon} =&
    p(y;\xi_{\out})^{\gamma} - \sum_c w_c p_c(y)^{\gamma},
\end{align}
and the derivative of the normalizing factor is 
\begin{align}
    \left. \frac{\partial z(w(\epsilon))}{\partial \epsilon}\right|_{\epsilon = 0} =&
    \int 
    \left\{ 
    p_{\out}^{\gamma} - \overline{p^{\gamma}}
    \right\} \dd \Lambda.
\end{align}
Then, 
\begin{align}
\notag
 {\rm IF}( p_\gamma(y), \xi_{\out})=&
\frac{1}{\gamma z(w)}
( p_{\out}^{\gamma} - \overline{p^{\gamma}})
\overline{p^{\gamma}}^{ \frac{1}{\gamma} -1}  \\ \notag
-&
\frac{1}{z^2(w)} \int (p_{\out}^{\gamma}-\overline{p^{\gamma}}) \dd \Lambda 
\overline{p^{\gamma}}^{\frac{1}{\gamma}} \\
=&
\frac{1}{\gamma} 
\frac{p_{\gamma}(y;w)}{\overline{p(y)^{\gamma}}}
(p_{\out}(y)^{\gamma} - \overline{p(y)^{\gamma}}) - 
\frac{p_{\gamma}(y;w)}{z(w)}
\int 
(p_{\out}(y)^{\gamma} - \overline{p(y)^{\gamma}}) \dd \Lambda(y).
\end{align}
We conclude that if $\gamma>0$, then there exists a limit of ${\rm IF}( q_\gamma, \xi_{\rm out})$ when $\xi_{\rm out}$ goes to $\infty$ or $-\infty$.
This is because $p(y;\xi_{\rm out})^\gamma$ converges to one of singular density functions $p_{\pm\infty}(y)^\gamma$ as 
$\xi_{\rm out}$ goes to $\infty$ or $-\infty$. 
On the other hand, if $\gamma\leq0$, the influence function becomes unbounded in the outlier.

\subsection{Proof of proposition~\ref{prop:IFac}}
\label{app:A5}
In the proof, we assume $\phi = 1$ and $c(y,\phi) = 0$ for the exponential family for the sake of simplicity. We also omit $\bx$ from the acquisition function. 
We first consider the case with the KL-divergence. The $\epsilon$-contamination for the acquisition function is denoted as
\begin{align} 
    a_0(\bx,w(\epsilon)) =&
    W_{\epsilon}[ D_0(p(y;\xi), p(y; W_{\epsilon}[\xi]))].
\end{align}
Taking the derivative of $a_0(\bx,w(\epsilon))$ with respect to $\epsilon$ and setting $\epsilon \to 0$, we have
\begin{align} \notag
    {\mathrm{IF}}(a_0(\bx;w),\xi_{\out}) =&
    \int p_{\out} (y \xi_{\out} - \psi(\xi_\out)) \dd \Lambda 
    -
    \int p_{\out} ( y \bar{\xi}-\psi(\bar{\xi})) \dd \Lambda \\ \notag 
    &- \sum_c w_c \int p_{c}(y\xi_c - \psi(\xi_c)) \dd \Lambda + 
    \sum_{c} w_c \int p_c (y \bar{\xi} - \psi(\bar{\xi})) \dd \Lambda \\ \notag 
    &- \sum_c w_c 
    \int p_c ( y \xi_{\out} - \psi^{\prime}(\xi_\out) - y \bar{\xi}+\psi^{\prime}(\bar{\xi})) \dd \Lambda \\  \notag
    =& 
    \left\{
    \psi^{\prime}(\xi_{\out}) - \bar{\xi} - \sum_c w_c \psi^{\prime}(\xi_c)
    \right\} \xi_{\out} \\ \notag
    &- \psi(\xi_{\out}) + \psi^{\prime}(\xi_{\out}) + 
    2 \sum_c w_c \psi^{\prime}(\xi_c) \bar{\xi}\\
    &- \sum_{c} w_c \psi^{\prime}(\xi_c) \xi_c + \sum_c w_c \psi(\xi_c) - \psi^{\prime}(\bar{\xi}).
\end{align}
Hence, the influence function is unbounded for $\xi_{\out}$.

Consider the influence function for 
\begin{align} \notag
a_{\beta}(\bx;w) =& \sum_c w_c D_{\beta}(p_c,p_u)
\\
=&
\frac{1}{\beta+1} \int p_u^{\beta+1}\dd \Lambda 
+
\sum_c w_c 
\left[
\frac{1}{\beta (\beta+1)}
\int p^{\beta+1}_c \dd \Lambda - 
\frac{1}{\beta} 
\int p_c p_u^{\beta} \dd \Lambda
\right],
\end{align}
where    
\begin{align}
    p_u(y;w) =
    \left[
    \sum_c w_c \exp (\beta (y \xi_c - \psi(\xi_c)) - \beta b
    \right]^{\frac{1}{\beta}}.
\end{align}

Taking the derivative of $a_{\beta}(\bx,w(\epsilon))$ with respect to $\epsilon$ and setting $\epsilon \to 0$, we have
\begin{align} \notag
{\mathrm{IF}}(a_{\beta}(\bx,w),\xi_{\out}) =&
\int 
\frac{1}{\beta} p_u^{\beta + \frac{1}{\beta} -1}
\left\{
p_{\out}^{\beta} - \overline{p^{\beta}}
\right\} \dd \Lambda 
-
\frac{1}{\beta} \int \overline{p^1} p_u^{\beta + \frac{1}{\beta} -2} 
(p_{\out}^{\beta} - \overline{p^{\beta}})\dd \Lambda \\ \notag
&+
\frac{1}{\beta (\beta+1)}
\int p_{\out}^{\beta +1} \dd \Lambda - \frac{1}{\beta} \int p_{\out} p_u^{\beta} \dd \Lambda \\
&- 
\frac{1}{\beta (\beta+1)}
\int \overline{p^{\beta+1}} \dd \Lambda +
\frac{1}{\beta} \int \overline{p^{1}} p_{u}^{\beta} \dd \Lambda.
\end{align}
We conclude that if $\beta >0$, then there exists a limit of $ \mathrm{IF}(a_u(\bx;w),\xi_{\out})$ when $|\xi_{\out}| \to \infty$.

For the dual $\gamma$-power divergence, the acquisition function is of the form
\begin{align} \notag
    a_{\gamma}(\bx;w) =&
    \sum_c w_c D^{\ast}_{\gamma}(p_c,p_{\gamma}) \\
    =&
    \sum_{c} w_c
    \left\{
    - 
    \frac{ \int p_c p_{\gamma}^{\gamma} \dd \Lambda}{
    \left(
    p_c^{\gamma +1} \dd \Lambda
    \right)^{\frac{1}{\gamma +1}}
    }
    +
    \left(
    \int p_{\gamma}^{\gamma+1}\dd \Lambda
    \right)^{\frac{\gamma}{\gamma+1}}
    \right\}.
\end{align}
The derivative of the $\epsilon$-contaminated model is
\begin{align}
    \left. \frac{\partial p_{\gamma}}{\partial \epsilon}\right|_{\epsilon =0} =
    \frac{1}{\gamma} \frac{p_{\gamma}}{\overline{p^{\gamma}}}
    (p_{\out}^{\gamma} - \overline{p^{\gamma}}) - 
    \frac{1}{z(w)} p_{\gamma} \int (p_{\out}^{\gamma} - \overline{p^{\gamma}}) \dd \Lambda.
\end{align}
Then, the influence function is
\begin{align} \notag
    {\mathrm{IF}}(a_{\gamma}(\bx;w),\xi_{\out}) 
    =&
    -
    \frac{
    \int p_{\out} p^{\gamma}_{\gamma} \dd \Lambda
    }{
    \left(
    \int p_{\out}^{\gamma+1} \dd \Lambda
    \right)^{\frac{1}{\gamma +1}}
    }
    +
    2
    \sum_{c} w_c 
    \frac{
    \int p_c p^{\gamma}_{\gamma} \dd \Lambda
    }{
    \left(
    \int p_c^{\gamma+1} \dd \Lambda
    \right)^{\frac{1}{\gamma +1}}
    } \\
    \notag 
    &- 
    \sum_c w_c \left( \int p_c^{\gamma+1} \dd \Lambda \right)^{-\frac{1}{\gamma+1}}
    \int 
    \frac{p_c p_{\gamma}^{\gamma}}{ \overline{p^{\gamma}}}
    p_{\out}^{\gamma} \dd \Lambda \\ \notag 
    &+
    \frac{1}{z(w)} p_{\gamma} 
    \int (p^{\gamma}_{\out} - \overline{p^{\gamma}} ) \dd \Lambda 
    \sum_c w_c \left(\int p_c^{\gamma+1} \dd \Lambda \right)^{- \frac{1}{\gamma+1}} \\
    &+
    \frac{p_{\gamma}^{\gamma+1}}{\overline{p^{\gamma}}} 
    (
    p_{\out}^{\gamma} - \overline{p^{\gamma}}
    ) \dd \Lambda 
    -
    \frac{\gamma}{z(w)} p_{\gamma} \int (p_{\out}^{\gamma} - \overline{p^{\gamma}}) \dd \Lambda.
\end{align}
We conclude that if $\gamma>0$, then there exists a limit of ${\rm IF}( a_\gamma(\bx;w), \xi_{\rm out})$ when $\xi_{\rm out}$ goes to $\infty$ or $-\infty$.

\end{document}